\documentclass[11pt]{article} 
\usepackage{fullpage}
\usepackage{times} 
\usepackage{courier}  
\usepackage[hyphens]{url}  
\usepackage{graphicx} 
\usepackage{natbib} 
\usepackage{caption} 

\usepackage[utf8]{inputenc} 
\usepackage[T1]{fontenc}    
\usepackage{hyperref}       
\usepackage{url}            
\usepackage{booktabs}       
\usepackage{amsfonts}       
\usepackage{microtype}      
\usepackage{xcolor}
\usepackage{enumerate}
\usepackage{enumitem}
\usepackage{amsthm}
\usepackage{physics}
\usepackage{amsmath}
\usepackage{tikz}
\usepackage{mathdots}
\usepackage{yhmath}
\usepackage{cancel}
\usepackage{color}
\usepackage{siunitx}
\usepackage{array}
\usepackage{multirow}
\usepackage{amssymb}
\usepackage{gensymb}
\usepackage{tabularx}
\usepackage{booktabs}

\usetikzlibrary{fadings}
\usetikzlibrary{patterns}
\usetikzlibrary{shadows.blur}
\usetikzlibrary{shapes}

\usepackage{algorithm}
\usepackage{algpseudocode}
\usepackage{amsmath}
\usepackage{graphics}
\usepackage{epsfig}

\newcommand{\E}{\mathbb{E}}
\newcommand{\PR}[1]{{\rm{\mathbb{P}}}({#1})}
\newcommand{\W}{\widetilde}

\newcommand{\states}{\mathcal{S}}
\newcommand{\actions}{\mathcal{A}}
\newcommand{\mat}[1]{\mathbf{#1}}

\newtheorem{theorem}{Theorem}
\newtheorem{lemma}{Lemma}
\newtheorem{corollary}{Corollary}

\newenvironment{itemize*}%
{\begin{itemize}[leftmargin=*,topsep=0pt]%
		\setlength{\itemsep}{0pt}%
		\setlength{\parskip}{0pt}}%
	{\end{itemize}}
\newenvironment{enumerate*}%
{\begin{enumerate}[leftmargin=*,topsep=0pt]%
		\setlength{\itemsep}{0pt}%
		\setlength{\parskip}{0pt}}%
	{\end{enumerate}}

\newtheorem{fact}{Fact}

\title{
A Provably Efficient Algorithm for Linear Markov Decision Process with Low Switching Cost \thanks{
Minbo Gao and Tianle Xie contribute equally in this paper.
	Correspondence to: Simon S. Du <ssdu@cs.washington.edu>, Lin F. Yang <linyang@ee.ucla.edu>
}
}

\author{
  Minbo Gao $^*$\\
  Tsinghua University\\
  \texttt{gmb17@mails.tsinghua.edu.cn} \\
  \and
  Tianle Xie $^*$\\
  Tsinghua University\\
  \texttt{xtl17@mails.tsinghua.edu.cn} \\
  \and
   Simon S. Du\\
  University of Washington\\
  \texttt{ssdu@cs.washington.edu} \\  
  \and
  Lin F. Yang\\
  University of California, Los Angeles\\
  \texttt{linyang@ee.ucla.edu} \\  
}

\begin{document}

\maketitle

\begin{abstract}
	Many real-world applications, such as those in medical domains, recommendation systems, etc, can be formulated as large state space reinforcement learning problems with only a small budget of the number of policy changes, i.e., low switching cost.
This paper focuses on the linear Markov Decision Process (MDP) recently studied in \cite{yang2019sample,jin2019provably} where the linear function approximation is used for generalization on the large state space.
We present the first algorithm for linear MDP with a low switching cost.
Our algorithm achieves an $\widetilde{O}\left(\sqrt{d^3H^4K}\right)$ regret bound with a near-optimal $O\left(d H\log K\right)$ \emph{global} switching cost where $d$ is the feature dimension, $H$ is the planning horizon and $K$ is the number of episodes the agent plays.
Our regret bound matches the best existing polynomial algorithm by \cite{jin2019provably} and our switching cost is exponentially smaller than theirs.
When specialized to tabular MDP, our switching cost bound improves those in \cite{bai2019provably,zhang2020optimal}.
We complement our positive result with an $\Omega\left(dH/\log d\right)$ global switching cost lower bound for any no-regret algorithm.

\end{abstract}

\section{Introduction}
\label{sec:intro}
Reinforcement learning (RL) is often used for modeling real-world sequential decision-making problems such as medical applications \citep{mahmud2018applications, istepanian2009medical}, personalized recommendation \citep{zheng2018drn, zhao2018deep}, hardware placements \citep{mirhoseini2017device}, database optimization \citep{krishnan2018learning}, etc.
For these applications, oftentimes it is desirable to restrict the agent from adjusting its policy frequently.
For instance, in medical domains, changing a policy requires a thorough approval process by experts \citep{lei2012smart, almirall2012designing, almirall2014introduction};
for large-scale software and hardware systems, changing a policy requires to change the physical environment significantly \cite{mirhoseini2017device}.
Formally, we would like our RL algorithm admits a \emph{low switching cost}.
In this setting, the agent is only allowed to switch its policy for at most $N$ times, where $N$ is much smaller than the total number of rounds played.

The problem of designing provably efficient RL algorithm with low switching cost was first studied in \cite{bai2019provably} where authors proposed a $Q$-learning algorithm with upper confidence bound (UCB) bonus for tabular RL problems.
Their algorithm  achieves both a low regret and a low switching cost. 
More detailed discussions are included in the related work part.
However, two major problems remain open in the field:
\begin{enumerate}
	\item \textbf{Large state space}: 
	\citet{bai2019provably} studied tabular RL settingand their bounds scale polynomially with the number of states.
	Many aforementioned applications have a large state space. In medical domains, the state space can be all possible combinations of features that describe a patient.
	In a large scale system, every configuration is one state.
	For these problems, we need to use function approximation to generalize across states and keep the switching cost low at the same time.
	\item \textbf{Global switching cost}:
	\citet{bai2019provably} studied the local switching cost, the sum of number of changes of the policy on the state in the episode.  
	However, arguably in most applications, we are more interested in the global switching cost, which is the number of policy changes.
	For example, for medical domains, the cost of changing the entire policy is similar to that of changing a component of the policy (the decision on a specific state).
	Furthermore, for RL problems with a large state space, it is natural to study the global switching cost because the local switching cost necessarily scales with the number of states, which is already large. On the hand, the global switching cost need not to scale with the number of states, and thus it is a more meaningful quantity to characterize.
\end{enumerate}

In this paper we tackle these two problems head-on in the linear Markov Decision Process (MDP) recently studied in \cite{yang2019sample,jin2019provably}, in which the linear function is used for generalization across states.
Our contributions are summarized below.
\begin{itemize}
	\item We present the first provably efficient algorithm for linear MDP with low switching cost.
	Our algorithm enjoys $\widetilde{O}\left(\sqrt{d^3H^4K}\right)$ regret and $O\left(dH \log K\right)$ global switching cost where $d$ is the feature dimension, $H$ is the planning horizon and $K$ is the number of episodes the agent plays.
	The regret bound matches the best existing polynomial algorithm by \cite{jin2019provably} and the switching cost is significantly lower.
	Furthermore, since tabular MDP is a special case of linear MDP, our result directly implies an $O\left(SAH\log K\right)$ global switching cost bound where $S$ is the number of states and $A$ is the number of actions.
	\item We provide an $\Omega\left(dH/\log d\right)$ global switching cost lower bound for Linear MDP. 
	To our knowledge, no previous work provides global switching cost lower bound for MDP, let alone linear MDP.
	For comparison, \citet{bai2019provably} derived a local switching cost lower bound but it only implies an $\Omega\left(A\right)$ global switching cost lower bound. 
	
\end{itemize}

\section{Related Work}
\label{sec:rel}
Here we discuss related theoretical works.
There is a long line of works studying the sample complexity of tabular reinforcement learning~\citep{kearns1999finite, kakade2003sample, singh1994upper, azar2013minimax, sidford2018variance, sidford2018near,agarwal2019optimality,zanette2019almost,li2020breaking,azar2017minimax,dann2015sample,dann2017unifying,dann2019policy,jin2018qlearning,strehl2006pac,zhang2020optimal,max2019nonasymptotic,zanette2019tighter,dong2019qlearning,wang2020long,zhang2020reinforcement}.
The state-of-the-art analysis shows that one can obtain $O\left(\sqrt{H^3SAK}\right)$  regret\footnote{This regret bound applies to the setting where the transition probabilities can be different at each level and the reward at each level is bounded by $1$.} and this is tight~\citep{dann2015sample,osband2016on,jin2018qlearning}.

However, for real-world problems, the state space is often large, so we need to use function approximation.
Developing provably efficient algorithms for large state space RL problems is a hot topic recently~\citep{wen2013efficient,li2011knows,du2019good,du2020agnostic,krishnamurthy2016pac,jiang2017contextual,dann2018oracle,du2019provably,sun2018model,du2019q,feng2020provably,yang2019sample,yang2019reinforcement,jin2019provably,zanette2019frequentist,zanette2020learning,wang2020provably,wang2019optimism,cai2019provably,ayoub2020model}.
These works are based on different assumptions.
Our paper follows the setup of linear MDP \citep{yang2019sample,jin2019provably} where linear function is used for generalizations on the transition and the reward. See Section~\ref{sec:pre} for the precise definition.
Our algorithm is inspired by the one in \cite{jin2019provably}, which is a polynomial time algorithm.
Recently, \citet{zanette2020learning} gave an algorithm which has better regret than the one in \cite{jin2019provably}, but it is not computationally efficient.\footnote{
\citet{zanette2020learning} actually only requires a low inherent Bellman error condition, which is weaker than the linear MDP assumption.
}

Low switching cost algorithms were first studied in the bandit setting~\citep{auer2002finite,cesa2013online}.
To our knowledege, \cite{bai2019provably} is the first work studying the switching cost problem in RL.
\citet{bai2019provably} focused on the local switching cost which is defied as $\sum_{k=1}^{K-1} |\{(h,x) \in [H] \times \mathcal{S}:$ $ \pi_k^h(x) \neq \pi_{k+1}^h(x)\}|$ 
where $\mathcal{S}$ is the state space and $\pi_k^h$ is the policy at the $h$-the level in the $k$-th episode.
Our paper focuses on the global switching cost (cf. Equation~\eqref{eqn:switching_cost}) which is often more natural in applications.
\citet{bai2019provably} provided an $O\left(H^3SA\log\left(K/A\right)\right)$ local switching upper bound.
The upper bound was improved to $O\left(H^2SA\log\left(K\right)\right)$ by \citet{zhang2020optimal}.
As direct corollary of our main result, we can obtain an $O\left(HSA\log K\right)$ local switching cost upper bound.
\citet{bai2019provably} also provided a $\Omega\left(HSA\right)$ local switching cost lower bound.
However, this lower bound only implies a trivial $\Omega\left(A\right)$ global switching cost lower bound.
In this paper we provide an $\Omega\left(dH/\log d\right)$ lower bound, which is the first non-trivial lower bound for the global switching cost in RL.

\section{Preliminaries}
\label{sec:pre}
\subsection{Notations} We use $\norm{\cdot}$ to denote the standard Euclidean norm.
Given a positive integer $N$, we let $[N]=\left\{1,2,\ldots,N\right\}$.
For a matrix $A$, we use $\det(A)$ 
to denote its determinant.
 For two symmetric matrices, $A$ and $B$, $A \preccurlyeq B$ means the matrix $B - A$ is positive semidefinite.
We use the standard $O(\cdot)$ and $\Omega(\cdot)$ notations to hide universal constant factors, and $\widetilde{O}$, and $\widetilde{\Omega}$ notations to hide logarithmic factors.

\subsection{Markov Decision Process}
Throughout our paper, we consider the episodic Markov decision model $( \mathcal{S}, \mathcal{A}, \mathrm{H}, \mathbb{P}, \mathrm{r} )$. 
In this model,  $\mathcal{S}$ and $\mathcal{A}$ denote the set of states and actions, respectively. 
All the episodes have the same number of transitions taking place, which we use $\mathrm{H} \in N$ to denote. $\mathbb{P} = \{ \mathbb{P}_h \}$ is the set of transition probability measures. Hence $\mathbb{P}_h(x'|x,a)$ means the transition probability of taking action $a$ at step $h\in [H] = \{1,2,\cdots, H \}$ on state $x$ to the state $x'$. 
$\mathrm{r}$ is a collection of reward functions $r_h:\mathcal{S}\times \mathcal{A}\to [0,1]$ for each step in an episode. 

The dynamics of the episodic MDP can be view as the interaction of an agent with the environment periodically. 
At the beginning of an episode $k$, an arbitrary state $x_1^k \in \states$ is selected by the environment, and the agent is then in step $1$. At each step $h$ in this episode, based on the current state $x_h^k \in \states$ and the history information, the agent needs to decide which action to take. After action $a_h^k \in \actions $ is chosen, the environment will give the reward for the step $r_h(x_h^k, a_h^k)$ and move the agent to the next state $x_{h+1} \in \states$. 
The episode automatically ends when the agent reaches the step $H+1$. In other words, the agent will take at most $H$ actions and receive corresponding rewards in each episode.

To clarify the choice of actions for the agent in the episode, we define the policy function $\pi: \mathcal{S} \times [H]\to \mathcal{A}$. Namely, $\pi (x,h)$ is the action taken on state $x$ at step $h$ by the agent. 
We use $Q$-function to evaluate the long-term value for the action $a$ and subsequent decisions. The $Q$-function is defined as follows:

\begin{equation}\label{Q-function}
\begin{split}
    Q_{h}^{\pi}(x, a) := r_{h}(x, a)
    + \mathbb{E}\left. \left[ \sum_{i=h+1}^{H} r_{i}\left(x_{i}, \pi\left(x_{i}, i \right)\right) \right| x_{h}=x, a_{h}=a\right]
\end{split}
\end{equation}
In addition, we define the value function $V_h^{\pi}: \mathcal{S}\to \mathbb{R} $ for the policy $\pi$ via the following formula:
\begin{equation}\label{value function}
    V_{h}^{\pi}(x):=\mathbb{E}\left. \left[\sum_{i=h}^{H} r_{i}\left(x_{i}, \pi\left(x_{i}, i\right)\right) \right| x_{h}=x \right].
\end{equation}
The $Q$-function and $V$-function obey the following Bellman equation: for any policy $\pi$,
\begin{align*}
    Q_{h}^{\pi}(x, a)=&\left(r_{h}+\mathbb{P}_{h} V_{h+1}^{\pi}\right)(x, a), \quad V_{h}^{\pi}(x)=Q_{h}^{\pi}\left(x, \pi_{h}(x)\right),
    \quad \text{and}\quad V_{H+1}^{\pi}(x)=0,
\end{align*}
where
\[
\left[\mathbb{P}_{h} V_{h+1}\right](x, a):=\mathbb{E}_{x^{\prime} \sim \mathbb{P}_{h}(\cdot | x, a)} V_{h+1}\left(x^{\prime}\right).
\]
We denote $V_h^{*} (x) = \sup_{\pi} V_h^{\pi}(x)$, $Q_h^*(x,a) = \sup_{\pi} Q_h^{\pi} (x,\pi(x))$ as the optimal value and $Q$-functions.
The Bellman equation also holds for $V_h^*$ and $Q_h^*$ with respect to the optimal policy $\pi^*$.

Suppose an agent is allowed to interact with the MDP for $K$ episodes and plays policy $\pi_k$ at episode $k\in[K]$. We use regret to measure the performance of its algorithm, which is 
the difference of the value of the optimal policy and the policy adopted by the agent. 
\begin{equation}
    \operatorname{Regret}(K)=\sum_{k=1}^{K}\left[V_{1}^{\star}\left(x_{1}^{k}\right)-V_{1}^{\pi_{k}}\left(x_{1}^{k}\right)\right]
\end{equation}

\subsection{Linear Markov Decision Process}
\label{LMDPassumption}

The focus of our study is the linear MDP model~\citep{yang2019sample,jin2019provably}.
Linearity here represents that the transition probability and the reward are linear functions given the feature map.
 Formally, there exists a map from the state-action space to the feature space, namely $\phi: \mathcal{S}\times \mathcal{A} \to \mathbb{R}^d$ and a measure $\mu_h$ for $h \in [H]$ such that $\forall (x,a)\in \mathcal{S}\times \mathcal{A}$
\begin{align*}
	\mathbb{P}_h (\cdot | x,  a) = \left< \phi(x, a),\mu_h (\cdot) \right>, \quad\text{and}\quad 
	r_h(x, a) = \left\langle \phi(x, a), \theta_h\right\rangle.
\end{align*}
We further assume that$ \| \phi(x,a) \| \le 1$, $\forall x\in \mathcal{S}, \|\mu_h(x)\|\le \sqrt{d}$, and $\|\theta_h\|\le \sqrt{d}$.
Linear MDP model is a strict generalization of the standard tabular RL model with $d = \abs{\states}\abs{\actions}$.
For each $\left(s,a\right) \in \states \times \actions$, we can let $\phi\left(s,a\right) = e_{(s,a)}$
be the canonical basis in $\mathbb{R}^d$.
Then we can just define $\left\langle e_{(s,a)},\mu_h\left(\cdot\right)\right\rangle= \mathbb{P}_h\left(\cdot \mid s,a\right)$ and $\left\langle e_{(s,a)}, \theta_h\right\rangle = r_h(s, a)$.
See \cite{yang2019sample,jin2019provably} for more examples.

\subsection{Switching Cost}

The concept of switching cost is used to quantify the adaptability of reinforcement learning algorithms.  
The main focus of our work is the global switching cost, which counts the number of policy changes   in the running of the algorithm in $K$ episodes, namely:
\begin{align}
	N_{\text{switch}} ^{\text{gl}} \triangleq \sum_{k = 1}^{K-1} \mathbb{I} \{ \pi_k \not = \pi_{k+1} \} \label{eqn:switching_cost}
\end{align}
The focus of \cite{bai2019provably} is local switching cost:
\begin{align}
N_{\text{switch}}^{\text{loc}} \triangleq\sum_{k=1}^{K-1} \left|\left\{(h,x) \in [H] \times \mathcal{S} : \pi_k^h(x) \neq \pi_{k+1}^h(x)\right\}\right| \label{eqn:switching_cost_loc}
\end{align}
Technically, we always have 
$$N_{\text{switch}} ^{\text{gl}} \le N_{\text{switch}}^{\text{loc}} \le \abs{\states}HN_{\text{switch}} ^{\text{gl}}. $$
One crucial reason to use the global switching cost is that the definition of local switching cost is based on the number of states, which can be infinite in the linear MDP model, so the global switching cost is more meaningful quantity to study.
Lastly, we emphasize that in the study of the switching cost, we only consider deterministic policies.
Note that, by the Bellman optimality equation, there exist at least one optimal policy that is deterministic.

\section{Algorithm and Result}
\label{sec:algoIntro}
In this section, we describe our main positive result.
We first describe our approach, which is listed in Algorithm ~\ref{algo:lowSwitchingAlgo}.
Our algorithm has two crucial components: a $Q$-function estimation step and a policy update step.
This estimation step largely follows the method in \cite{jin2019provably}.
To achieve low-switching cost, the planning step is novel.
We use the determinant of the feature covariance matrix to guard the change of the policies.

More formally, in line~\ref{lin:estimate_start} - line~\ref{lin:estimate_end}, we use UCB to obtain an optimistic estimate of the optimal $Q$-function.
In Line~\ref{lin:estimate_start}, we define $\Lambda_h^k$ to be the empirical covariance matrix based on all features collected at level $h$, in which  a small regularization term $\lambda \mat{I}$ is added to avoid degeneracy.
In Line~\ref{lin:ls}, we use least-square to estimate the linear coefficient, where we construct labels as $r_{h}\left(x_{h}^{\tau}, a_{h}^{\tau}\right)+\max _{a} \widetilde Q_{h+1}^k \left(x_{h+1}^{\tau}, a\right)$, following the Bellman Equation.
In Line~\ref{lin:estimate_end}, we define our optimistic estimate of $Q$-function as the summation of a linear function $(\mathbf{w}_{h}^k)^{\top} \boldsymbol{\phi}(\cdot, \cdot)$ and a bonus term $\beta\left[\boldsymbol{\phi}(\cdot, \cdot)^{\top}  (\Lambda_{h}^k)^{-1} \boldsymbol{\phi}(\cdot, \cdot)\right]^{1 / 2}$.
The bonus term ensures our estimate is optimistic (cf.~Equation\eqref{B.5}).
We also clips the value to $H$ if it is too large.
We refer readers to \cite{jin2019provably} for more intuitions about this estimation procedure.

For policy update, the policy at the $k$-th episode in \cite{jin2019provably} is just to choose the action that maximizes the optimistic estimate $\widetilde{Q}^{k}_h$.
Since $\widetilde{Q}^{k}_h$ is changing at every episode, the policy changes at every episode as well.
Therefore, the switching cost can be linear in the number of episode.
Our main technique to reduce the switching cost is a new criteria to decide whether to update the policy.
More specifically, note when executing the policy, we always choose the action according to $Q_h^k$ (cf. Equation~\eqref{lin:take_action}) and $Q_h^k$ is updated according to $\widetilde{Q}_h^{\tilde{k}}$ in line~\ref{lin:update_q}, where $\widetilde{Q}_h^{\tilde{k}}$ is a reference $Q$-function estimate which changes infrequently.
Note that, as we will show shortly, $Q_h^k$ does not change frequently.
We use $\tilde{k}$ as a reference counter, which is updated  only in line~\ref{lin:change_reference} when the criteria in line~\ref{lin:criteria} is met.

Now we explain our proposed criteria.
At a high-level, since the empirical co-variance matrix $\Lambda_h^k$ determines both our estimate of $Q$-function and the bonus, if it changes a significant amount, this means we already learned new information and we need to change the policy to achieve low regret.
Note this step is computationally efficient because we just need to check the least eigenvalue of $2 (\Lambda_h^{k})^{-1} - (\Lambda_h^{\tilde{k}})^{-1}  $ is non-negative or not.
Geometrically, $(\Lambda_h^{\tilde{k}})^{-1}  \not \preccurlyeq 2 (\Lambda_h^{k})^{-1}$ represents that,  at the $k$th episode, there exists one direction at which we have learned twice information as the information we learned at the reference episode $\tilde{k}$.
We will explain more technical reasons in the next section.

\begin{algorithm}[t]
\caption{ Algorithm for Linear MDP with Low Global Switching Cost \label{algo:lowSwitchingAlgo}
}
\begin{algorithmic}[1]
\State \textbf{Input:} regularization parameter $\lambda > 0$.
\State Set $\tilde{k} \leftarrow1$.
\For{episode $k = 1, 2, \cdots, K$}
\Statex\textit{$Q$-function Estimation}
\For{step $h = H, \cdots, 1$}
\State $\Lambda_h^k \leftarrow \sum_{\tau=1}^{k-1} \boldsymbol{\phi}\left(x_{h}^{\tau}, a_{h}^{\tau}\right) \boldsymbol{\phi}\left(x_{h}^{\tau}, a_{h}^{\tau}\right)^{\top}+ \lambda \cdot \mathbf{I}$  \label{lin:estimate_start}
\State $\mathbf{w}_{h}^k \leftarrow  (\Lambda_{h}^k)^{-1} \sum_{\tau=1}^{k-1} \boldsymbol{\phi}\left(x_{h}^{\tau}, a_{h}^{\tau}\right)\Big[ r_{h}\left(x_{h}^{\tau}, a_{h}^{\tau}\right) +\max _{a} \widetilde Q_{h+1}^k \left(x_{h+1}^{\tau}, a\right)\Big]$ \label{lin:ls}
\State $\widetilde Q_{h}^{k}(\cdot, \cdot) \leftarrow \min \Big\{(\mathbf{w}_{h}^k)^{\top} \boldsymbol{\phi}(\cdot, \cdot) + \beta\left[\boldsymbol{\phi}(\cdot, \cdot)^{\top}  (\Lambda_{h}^k)^{-1} \boldsymbol{\phi}(\cdot, \cdot)\right]^{1 / 2}, H\Big\}$ \label{lin:estimate_end}
\Statex\textit{Policy Update}
\If{ $  (\Lambda_h^{\tilde{k}})^{-1}  \not \preccurlyeq 2 (\Lambda_h^{k})^{-1}$ } \label{lin:criteria}
\State Set $\tilde{k} \leftarrow k$ \label{lin:change_reference}
\EndIf
\State Set $Q_h^k \leftarrow \widetilde Q_h^{\tilde{k}}$ \label{lin:update_q}
\EndFor
\EndFor
\Statex\textit{Policy Execution}
\State Receive the initial state $x_1^k.$
\For{step $h = 1, 2, \cdots, H$}
\State Take action $a_{h}^{k} \leftarrow \arg \max_{a} Q_h^k\left(x_h^k,a\right)$.\label{lin:take_action}
\State Observe $x_{h+1}^k$.
\EndFor
\end{algorithmic}
\end{algorithm}

We now state the main theorem on the  regret bound and the switching cost bound of our algorithm.
\begin{theorem}[Regret and Switching Cost of Algorithm~\ref{algo:lowSwitchingAlgo} for Linear MDP]\label{thm:main}
In the linear  MDP setting, there exists a constant $c > 0$ such that, for any $p\in (0,1)$ , if we set $\lambda = 1$ and $\beta = c dH\sqrt{\iota}$  with $\iota = \log (2dKH/p)$ in Algorithm~\ref{algo:lowSwitchingAlgo}, then with probability $ 1 - p$, the total regret is at most $O(\sqrt{d^3 H^4 K \iota^2})$. Furthermore, the global switching cost of the algorithm is bounded by $O( dH\log K)$.
\end{theorem}

Theorem~\ref{thm:main} suggests our algorithm achieves the desired regret and switching cost guarantees.
In terms of the regret, our bound matches the one in \cite{jin2019provably}, but our algorithm has significantly lower switching cost ($O\left(dH \log K\right)$ v.s.  $K$).
Recently, \cite{zanette2020learning} gave an $\widetilde{O}\left(\sqrt{d^2H^3K}\right)$ regret bound but their algorithm is not computationally efficient.
An interesting open problem is to design an algorithm which enjoys a regret bound of  $\widetilde{O}\left(\sqrt{d^2H^4 K}\right)$  and a switching cost bound like ours.
As will be seen in Section~\ref{sec:lb}, our switching cost bound is near-optimal up to logarithmic factors.
We provide a proof sketch in Section~\ref{sec:proof_sketch} and defer the full proof to appendix.


Recall tabular MDP is special case of linear MDP.
Using the observation that in the tabular setting, whenever Algorithm~\ref{algo:lowSwitchingAlgo} changes the policy, it only change one state-action pair, we obtain the following result for the local switching cost.
\begin{corollary}[Regret and Switching Cost of Algorithm~\ref{algo:lowSwitchingAlgo} for Tabular MDP]\label{cor:tabular}
In the tabular setting, there exists a constant $c > 0$ such that, for any $p\in (0,1)$ , if we set $\lambda = 1$ and $\beta = c dH\sqrt{\iota}$  with $\iota = \log (2SAKH/p)$ in Algorithm~\ref{algo:lowSwitchingAlgo} , then with probability $ 1 - p$, the total regret is at most $O(\sqrt{S^3A^3 H^4 K \iota^2})$. Furthermore, the local switching cost of the algorithm is bounded by $O( SAH\log K)$.
\end{corollary}

We present our corollary in terms of the local switching cost in order to have a fair comparison with the results in \cite{bai2019provably,zhang2020optimal}.
Recall the local switching cost is always an upper bound of the global switching cost, so our bound also holds for the global switching cost.
The best existing result is by \citet{zhang2020optimal} who designed an algorithm with $\widetilde{O}\left(\sqrt{SAH^3 K}\right)$ regret and $O\left(SAH^2 \log K\right)$ switching cost.
Comparing with \cite{zhang2020optimal}, our regret bound is larger but our switching cost is lower than theirs.


\section{Proof Sketch of Theorem~\ref{thm:main}}
\label{sec:proof_sketch}
The proof consists of two parts: bounding the regret and bounding the global switching cost. 
Note minimizing the regret and the switching cost are conflict to each other because a small switching cost requires us \emph{not} to use the most updated information which can incur higher regret.
The main technical novelty is that our criteria can achieve the same order regret as the one in \cite{jin2019provably} and at the same time reduce the switching cost significantly.

\subsection{Regret Analysis}
Due to the delayed policy update, the establishment of the bound for regret may be more difficult than the previous algorithm in \cite{jin2019provably}, yet the steps are very similar. 

We start our proof by decomposing the regret into the error induced by the estimation error from the delayed policy update.
To simplify the notation, for any $k \in [K]$, we let $\W k \le k$ represents the episode index we update the policy to the one used in the $k$-th episode.
We have the following decomposition.
\begin{align*}
    \text{Regret}(K) & \le  \sum_{k = 1}^{K} [\widetilde V_1^{\tilde{k}}(x_1^k) - V_1^{\pi_{\tilde{k}}}(x_1^k)] 
\end{align*}
By definition, this term represents the error from the estimation in episode $\W k$ in state $x_1^k$.

\paragraph{Analysis of Error due to the Delayed Policy Update}
First, as will be seen in the appendix, we can obtain a recursive formula such that it is sufficient to bound the term 
$\left(\W {Q}_{h}^{\tilde{k}}-Q_{h}^{\pi_{\tilde{k}}}\right)\left(x_{h}^{k}, a_{h}^{k}\right)$.
With some error analysis, we can bound it by 

\[
\left(\W {Q}_{h}^{\tilde{k}}-Q_{h}^{\pi_{\tilde{k}}}\right)\left(x_{h}^{k}, a_{h}^{k}\right)\le \Delta_{h}^{\tilde{k}}(x, a)+\W {\delta}_{h}^{k}
\]

Here $\W {\delta}_{h}^{k}$ is a zero-mean martingale difference sequence, so we can use standard concentration inequalities to bound it.
$\Delta_{h}^{k}(x, a) $ represents the bonus term that satisfies \[\abs{\Delta_{h}^{k}(x, a)} \le \beta \sqrt{\phi(x, a)^{\top}\left(\Lambda_{h}^{k}\right)^{-1} \phi(x, a)}.\]
By our policy update criteria, we have the following simple yet crucial property: 
\[
\phi^\top (\Lambda_h^{\W k})^{-1} \phi \le 2 \phi^\top (\Lambda_h^{k})^{-1} \phi .\]
Therefore, although the error due to the delayed update has additional terms, these terms are at most of the same order as the error occurring during the estimation phase.
Using this observation, we can essentially reuse the proof for bounding the error due the estimation here.

\subsection{Switching Cost Analysis}

The analysis of switching cost is trickier.
We employ a potential based analysis.
The potential function is the logarithm of the determinant of the empirical covariance matrix.
The following lemma shows it is upper bounded by $O\left(d \log K\right)$.

\begin{lemma}\label{detlemma_2}
	Let $\phi_{\tau}$ are $d$-dimensional vectors satisfying $\norm{\phi_{\tau}} \le 1$. Let $A = \sum_{\tau = 1}^{K}\phi_{\tau}(\phi_{\tau})^{\top} + \lambda \cdot \mathrm{I}.$ Then we have \[\log \det A = O(d\log K).\]
\end{lemma}

Now we consider our update rule.
Recall we update our policy only if $(\Lambda_h^k)^{-1}  \not \preccurlyeq 2 (\Lambda_h^{\tilde{k}})^{-1}$.
The following lemma shows whenever this condition holds, the potential function must increase by a constant.

%

\begin{lemma}\label{detlemma}
	Assume $m\le n$, $A = \sum_{\tau = 1}^{m} \phi_{\tau}\phi_{\tau}^{\top} + \lambda \cdot \mathrm{I}$, $B = \sum_{\tau = 1}^{n} \phi_{\tau}\phi_{\tau}^{\top} + \lambda \cdot \mathrm{I}$. Then if $A^{-1} \not \preccurlyeq  2B^{-1}$, we have \[\log \det B \ge \log \det A + \log 2\]
\end{lemma}

To bound the switching cost, we note the potential is upper bounded by $O\left(d \log K\right)$ and every time we update the policy, the potential must increase by $\log 2$, so in total we at most update the policy $O\left(d \log K\right)$ times.
We believe our proof strategy may be useful in other problems as well.


\section{Lower Bound}
\label{sec:lb}
To  complement our upper bound on linear MDP, we present the following lower bound on the global  switching cost.

\begin{theorem}\label{thm:lowerBoundThm}
  For $d \ge 100$,  let $\mathcal{M}$ be  the class of linear MDPs defined in Section~\ref{LMDPassumption}.
For any algorithm that uses a deterministic policy at each episode, if its global switching cost $N_{\rm switch}^{\rm gl} \le \frac{dH}{100 \log d}$ , we have
\[\sup_{M \in \mathcal{M}} \mathbb{E}_{s_1, \mathcal{M}}{\left[ \sum_{k = 1}^{K} V_1^*(s_1) - V_1^{\pi_k}(s_1) \right ]}  \ge KH/4. \]
\end{theorem}
Theorem~\ref{thm:lowerBoundThm} states that for any algorithm that achieves sub-linear regret, it must have at least $\Omega\left(dH/\log d \right)$ global switching.
This shows our upper bound on the global switching cost cannot be improved up to logarithmic factors.
One interesting open problem is to further close this gap.
We remark that \citet{bai2019provably} derived an $\Omega\left(SAH\right)$ local switching cost, which only implies an $\Omega\left(A\right)$ global switching cost.
The simple multi-armed bandit problem also has an $\Omega\left(A\right)$ 
global switching cost lower bound.
Theorem~\ref{thm:lowerBoundThm} is, to our knowledge, the first non-trivial global switching cost lower bound in RL.

\subsection{Proof Sketch of Theorem~\ref{thm:lowerBoundThm}}
The full proof is deferred to the appendix, and here we give an outline of the proof.
The  strategy is to construct a class of hard MDPs and show for any algorithm without any prior knowledge about this class, it must suffer enough regret and switching cost.
The difficult part is how to construct hard instances.

%
%

\begin{figure}

\centering

\tikzset{every picture/.style={line width=0.75pt}} 

\begin{tikzpicture}[x=0.75pt,y=0.75pt,yscale=-0.8,xscale=0.8]

\draw  [color={rgb, 255:red, 0; green, 0; blue, 0 }  ,draw opacity=1 ][dash pattern={on 4.5pt off 4.5pt}] (100,145) .. controls (100,131.19) and (111.19,120) .. (125,120) .. controls (138.81,120) and (150,131.19) .. (150,145) .. controls (150,158.81) and (138.81,170) .. (125,170) .. controls (111.19,170) and (100,158.81) .. (100,145) -- cycle ;
\draw  [color={rgb, 255:red, 0; green, 0; blue, 0 }  ,draw opacity=1 ] (250,145) .. controls (250,131.19) and (261.19,120) .. (275,120) .. controls (288.81,120) and (300,131.19) .. (300,145) .. controls (300,158.81) and (288.81,170) .. (275,170) .. controls (261.19,170) and (250,158.81) .. (250,145) -- cycle ;
\draw  [color={rgb, 255:red, 0; green, 0; blue, 0 }  ,draw opacity=0.43 ] (400,145) .. controls (400,131.19) and (411.19,120) .. (425,120) .. controls (438.81,120) and (450,131.19) .. (450,145) .. controls (450,158.81) and (438.81,170) .. (425,170) .. controls (411.19,170) and (400,158.81) .. (400,145) -- cycle ;
\draw [color={rgb, 255:red, 0; green, 0; blue, 0 }  ,draw opacity=1 ]   (157.5,144.25) -- (243.5,144.25) ;
\draw [shift={(245.5,144.25)}, rotate = 180] [color={rgb, 255:red, 0; green, 0; blue, 0 }  ,draw opacity=1 ][line width=0.75]    (10.93,-3.29) .. controls (6.95,-1.4) and (3.31,-0.3) .. (0,0) .. controls (3.31,0.3) and (6.95,1.4) .. (10.93,3.29)   ;
\draw [color={rgb, 255:red, 0; green, 0; blue, 0 }  ,draw opacity=1 ]   (107,122.75) .. controls (78.78,85.62) and (174.06,86.23) .. (143.47,122.15) ;
\draw [shift={(142.5,123.25)}, rotate = 312.15999999999997] [color={rgb, 255:red, 0; green, 0; blue, 0 }  ,draw opacity=1 ][line width=0.75]    (10.93,-3.29) .. controls (6.95,-1.4) and (3.31,-0.3) .. (0,0) .. controls (3.31,0.3) and (6.95,1.4) .. (10.93,3.29)   ;
\draw    (126,174.25) .. controls (163.62,210.88) and (382.56,207.33) .. (425.75,178.14) ;
\draw [shift={(427,177.25)}, rotate = 503.13] [color={rgb, 255:red, 0; green, 0; blue, 0 }  ][line width=0.75]    (10.93,-3.29) .. controls (6.95,-1.4) and (3.31,-0.3) .. (0,0) .. controls (3.31,0.3) and (6.95,1.4) .. (10.93,3.29)   ;
\draw [color={rgb, 255:red, 0; green, 0; blue, 0 }  ,draw opacity=1 ]   (256,120.75) .. controls (227.78,83.62) and (323.06,84.23) .. (292.47,120.15) ;
\draw [shift={(291.5,121.25)}, rotate = 312.15999999999997] [color={rgb, 255:red, 0; green, 0; blue, 0 }  ,draw opacity=1 ][line width=0.75]    (10.93,-3.29) .. controls (6.95,-1.4) and (3.31,-0.3) .. (0,0) .. controls (3.31,0.3) and (6.95,1.4) .. (10.93,3.29)   ;
\draw [color={rgb, 255:red, 0; green, 0; blue, 0 }  ,draw opacity=1 ]   (407,120.75) .. controls (378.79,83.62) and (474.06,84.23) .. (443.47,120.15) ;
\draw [shift={(442.5,121.25)}, rotate = 312.15999999999997] [color={rgb, 255:red, 0; green, 0; blue, 0 }  ,draw opacity=1 ][line width=0.75]    (10.93,-3.29) .. controls (6.95,-1.4) and (3.31,-0.3) .. (0,0) .. controls (3.31,0.3) and (6.95,1.4) .. (10.93,3.29)   ;

\draw (136,145) node   [align=left] {\begin{minipage}[lt]{24.48pt}\setlength\topsep{0pt}
{\Large $\displaystyle \textcolor[rgb]{0,0,0}{u}$}
\end{minipage}};
\draw (282,145) node   [align=left] {\begin{minipage}[lt]{18.360000000000003pt}\setlength\topsep{0pt}
{\Large $\displaystyle \textcolor[rgb]{0,0,0}{v}$}
\end{minipage}};
\draw (430,145) node   [align=left] {\begin{minipage}[lt]{18.360000000000003pt}\setlength\topsep{0pt}
{\Large $\displaystyle w$}
\end{minipage}};

\end{tikzpicture}
\caption{An illustration of the hard instance used for proving Theorem~\ref{thm:lowerBoundThm}.}
\label{fig:comblock}
\end{figure}
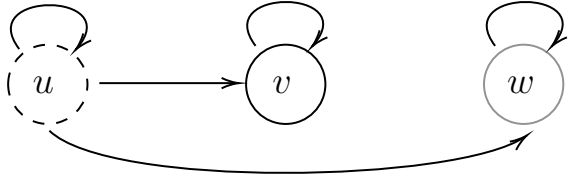

We consider environments similar to combination lock~\citep{kakade2003sample}.
Figure \ref{fig:comblock}  shows a simplified version of our constructed environment. The agent starts at state $u$, and $v, w$ are other two states. The reward at $v$ is always $1$ while the reward at other states is $0$. In order to go to state $v$, the agent needs to select a sequence of correct actions. 
In each episode, the agent stays at $u$ if the previous action is correct and even if only one action is incorrect, the agent will go to state $w$ at and stay there till the episode ends.  

We further construct states and features to encode this problem as a linear MDP.
To ensure the transition and the reward is linear we also need to adjust states and action carefully for which some auxiliary states are needed as well.
The $\log d$ factor in the denominator comes from our modifications of the environment described above.

\section{Conclusion and Future Works}
\label{sec:con}
In the view of the switching cost, we study the reinforcement learning algorithms in the linear Markov decision process setting.  Based on the current polynomial switching-cost algorithm with small regret bound, we design a new algorithm that matches its regret bound, with significantly lower global switching cost.  The regret bound of our algorithm is regret $\Tilde{O}(\sqrt{d^3H^4K})$, with the global switching cost being $O(dH\log K)$. 
This bound also implies a switching cost improvement over existing results of tabular MDP. 
Furthermore, by constructing a series of hard MDP instances, we are able to prove the lower bound for the switching cost is $\Omega(dH/\log d)$ provided that the deterministic algorithm has a sub-linear regret. 
We now list some future directions.

\paragraph{Towards Optimal Switching Cost Bound}
Currently, there is a $\log K \log d$ factor gap between our upper bound and lower bound.
In particular, we believe the upper bound can be further improved to $\log \log K$, as in the bandit setting, this is achievable~\cite{cesa2013online}.
For the lower bound, we believe the $\log d$ factor is removable though we found this is a technically challenging problem
We believe obtaining the optimal switching cost bound will greatly broaden our understanding on this problem.

\paragraph{Optimal Regret Bound with Switching Cost}
\citet{zanette2020learning} recently showed for linear MDP, it is possible to obtain an $\widetilde{O}\left(\sqrt{d^2H^3K}\right)$, which is optimal up to logarithmic factors.
Their algorithm is substantially different from the one by \cite{jin2019provably} and it is not computationally efficient.
It is possible to combine their analysis and ours to obtain an algorithm which has near-optimal regret and at the same time, has low switching cost.
A more interesting problem is to make this algorithm computationally efficient.

\paragraph{Low Switching Cost Algorithm for RL with General Function Approximation}
Recently, there are many works trying to design provably efficient algorithms  with general function approximation, beyond the linear function approximation scheme.
These works are based on different assumptions~\citep{wen2013efficient,jiang2017contextual,sun2018model,wang2020provably,ayoub2020model}.
It would be interesting to extend our analysis to these settings.

This work does not present any foreseeable negative societal consequence. From the positive side, the algorithm proposed in this paper can be potentially applied in medical domain, and hence benefit the society.
\bibliography{ref}

\begin{thebibliography}{54}
\providecommand{\natexlab}[1]{#1}
\providecommand{\url}[1]{\texttt{#1}}
\expandafter\ifx\csname urlstyle\endcsname\relax
  \providecommand{\doi}[1]{doi: #1}\else
  \providecommand{\doi}{doi: \begingroup \urlstyle{rm}\Url}\fi

\bibitem[Agarwal et~al.(2019)Agarwal, Kakade, and Yang]{agarwal2019optimality}
Alekh Agarwal, Sham Kakade, and Lin~F Yang.
\newblock On the optimality of sparse model-based planning for {Markov}
  decision processes.
\newblock \emph{arXiv preprint arXiv:1906.03804}, 2019.

\bibitem[Almirall et~al.(2012)Almirall, Compton, Gunlicks-Stoessel, Duan, and
  Murphy]{almirall2012designing}
Daniel Almirall, Scott~N Compton, Meredith Gunlicks-Stoessel, Naihua Duan, and
  Susan~A Murphy.
\newblock Designing a pilot sequential multiple assignment randomized trial for
  developing an adaptive treatment strategy.
\newblock \emph{Statistics in medicine}, 31\penalty0 (17):\penalty0 1887--1902,
  2012.

\bibitem[Almirall et~al.(2014)Almirall, Nahum-Shani, Sherwood, and
  Murphy]{almirall2014introduction}
Daniel Almirall, Inbal Nahum-Shani, Nancy~E Sherwood, and Susan~A Murphy.
\newblock Introduction to smart designs for the development of adaptive
  interventions: with application to weight loss research.
\newblock \emph{Translational behavioral medicine}, 4\penalty0 (3):\penalty0
  260--274, 2014.

\bibitem[Auer et~al.(2002)Auer, Cesa-Bianchi, and Fischer]{auer2002finite}
Peter Auer, Nicolo Cesa-Bianchi, and Paul Fischer.
\newblock Finite-time analysis of the multiarmed bandit problem.
\newblock \emph{Machine learning}, 47\penalty0 (2-3):\penalty0 235--256, 2002.

\bibitem[Ayoub et~al.(2020)Ayoub, Jia, Csaba, Wang, and Yang]{ayoub2020model}
Alex Ayoub, Zeyu Jia, Szepesvari Csaba, Mengdi Wang, and Lin~F. Yang.
\newblock Model-based reinforcement learning with value-targeted regression.
\newblock \emph{arXiv preprint arXiv:2006.01107}, 2020.

\bibitem[Azar et~al.(2013)Azar, Munos, and Kappen]{azar2013minimax}
Mohammad~Gheshlaghi Azar, R{\'e}mi Munos, and Hilbert~J Kappen.
\newblock Minimax {PAC} bounds on the sample complexity of reinforcement
  learning with a generative model.
\newblock \emph{Machine learning}, 91\penalty0 (3):\penalty0 325--349, 2013.

\bibitem[Azar et~al.(2017)Azar, Osband, and Munos]{azar2017minimax}
Mohammad~Gheshlaghi Azar, Ian Osband, and R{\'e}mi Munos.
\newblock Minimax regret bounds for reinforcement learning.
\newblock In \emph{Proceedings of the 34th International Conference on Machine
  Learning-Volume 70}, pages 263--272. JMLR. org, 2017.

\bibitem[Bai et~al.(2019)Bai, Xie, Jiang, and Wang]{bai2019provably}
Yu~Bai, Tengyang Xie, Nan Jiang, and Yu-Xiang Wang.
\newblock Provably efficient q-learning with low switching cost.
\newblock In \emph{Advances in Neural Information Processing Systems}, pages
  8002--8011, 2019.

\bibitem[Cai et~al.(2019)Cai, Yang, Jin, and Wang]{cai2019provably}
Qi~Cai, Zhuoran Yang, Chi Jin, and Zhaoran Wang.
\newblock Provably efficient exploration in policy optimization.
\newblock \emph{arXiv preprint arXiv:1912.05830}, 2019.

\bibitem[Cesa-Bianchi et~al.(2013)Cesa-Bianchi, Dekel, and
  Shamir]{cesa2013online}
Nicolo Cesa-Bianchi, Ofer Dekel, and Ohad Shamir.
\newblock Online learning with switching costs and other adaptive adversaries.
\newblock In \emph{Advances in Neural Information Processing Systems}, pages
  1160--1168, 2013.

\bibitem[Dann and Brunskill(2015)]{dann2015sample}
Christoph Dann and Emma Brunskill.
\newblock Sample complexity of episodic fixed-horizon reinforcement learning.
\newblock In \emph{Advances in Neural Information Processing Systems}, pages
  2818--2826, 2015.

\bibitem[Dann et~al.(2017)Dann, Lattimore, and Brunskill]{dann2017unifying}
Christoph Dann, Tor Lattimore, and Emma Brunskill.
\newblock Unifying {PAC} and regret: Uniform {PAC} bounds for episodic
  reinforcement learning.
\newblock In \emph{Proceedings of the 31st International Conference on Neural
  Information Processing Systems}, NIPS’17, page 5717–5727, Red Hook, NY,
  USA, 2017. Curran Associates Inc.
\newblock ISBN 9781510860964.

\bibitem[Dann et~al.(2018)Dann, Jiang, Krishnamurthy, Agarwal, Langford, and
  Schapire]{dann2018oracle}
Christoph Dann, Nan Jiang, Akshay Krishnamurthy, Alekh Agarwal, John Langford,
  and Robert~E. Schapire.
\newblock On oracle-efficient {PAC} {RL} with rich observations.
\newblock In \emph{NeurIPS}, 2018.

\bibitem[Dann et~al.(2019)Dann, Li, Wei, and Brunskill]{dann2019policy}
Christoph Dann, Lihong Li, Wei Wei, and Emma Brunskill.
\newblock Policy certificates: Towards accountable reinforcement learning.
\newblock In \emph{Proceedings of the 36th International Conference on Machine
  Learning}, volume~97 of \emph{Proceedings of Machine Learning Research},
  pages 1507--1516, Long Beach, California, USA, 09--15 Jun 2019. PMLR.

\bibitem[Dong et~al.(2019)Dong, Wang, Chen, and Wang]{dong2019qlearning}
Kefan Dong, Yuanhao Wang, Xiaoyu Chen, and Liwei Wang.
\newblock Q-learning with ucb exploration is sample efficient for
  infinite-horizon mdp, 2019.

\bibitem[Du et~al.(2019{\natexlab{a}})Du, Kakade, Wang, and Yang]{du2019good}
Simon~S Du, Sham~M Kakade, Ruosong Wang, and Lin~F Yang.
\newblock Is a good representation sufficient for sample efficient
  reinforcement learning?
\newblock \emph{arXiv preprint arXiv:1910.03016}, 2019{\natexlab{a}}.

\bibitem[Du et~al.(2019{\natexlab{b}})Du, Krishnamurthy, Jiang, Agarwal,
  Dud{\'\i}k, and Langford]{du2019provably}
Simon~S Du, Akshay Krishnamurthy, Nan Jiang, Alekh Agarwal, Miroslav
  Dud{\'\i}k, and John Langford.
\newblock Provably efficient {RL} with rich observations via latent state
  decoding.
\newblock \emph{arXiv preprint arXiv:1901.09018}, 2019{\natexlab{b}}.

\bibitem[Du et~al.(2019{\natexlab{c}})Du, Luo, Wang, and Zhang]{du2019q}
Simon~S Du, Yuping Luo, Ruosong Wang, and Hanrui Zhang.
\newblock Provably efficient {Q}-learning with function approximation via
  distribution shift error checking oracle.
\newblock In \emph{Advances in Neural Information Processing Systems}, pages
  8058--8068, 2019{\natexlab{c}}.

\bibitem[Du et~al.(2020)Du, Lee, Mahajan, and Wang]{du2020agnostic}
Simon~S Du, Jason~D Lee, Gaurav Mahajan, and Ruosong Wang.
\newblock Agnostic {Q}-learning with function approximation in deterministic
  systems: Tight bounds on approximation error and sample complexity.
\newblock \emph{arXiv preprint arXiv:2002.07125}, 2020.

\bibitem[Feng et~al.(2020)Feng, Wang, Yin, Du, and Yang]{feng2020provably}
Fei Feng, Ruosong Wang, Wotao Yin, Simon~S Du, and Lin~F Yang.
\newblock Provably efficient exploration for {RL} with unsupervised learning.
\newblock \emph{arXiv preprint arXiv:2003.06898}, 2020.

\bibitem[Istepanian et~al.(2009)Istepanian, Philip, and
  Martini]{istepanian2009medical}
Robert~SH Istepanian, Nada~Y Philip, and Maria~G Martini.
\newblock Medical qos provision based on reinforcement learning in ultrasound
  streaming over 3.5 g wireless systems.
\newblock \emph{IEEE Journal on Selected areas in Communications}, 27\penalty0
  (4):\penalty0 566--574, 2009.

\bibitem[Jiang et~al.(2017)Jiang, Krishnamurthy, Agarwal, Langford, and
  Schapire]{jiang2017contextual}
Nan Jiang, Akshay Krishnamurthy, Alekh Agarwal, John Langford, and Robert~E
  Schapire.
\newblock Contextual decision processes with low {Bellman} rank are
  {PAC}-learnable.
\newblock In \emph{Proceedings of the 34th International Conference on Machine
  Learning-Volume 70}, pages 1704--1713. JMLR. org, 2017.

\bibitem[Jin et~al.(2018)Jin, Allen-Zhu, Bubeck, and Jordan]{jin2018qlearning}
Chi Jin, Zeyuan Allen-Zhu, Sebastien Bubeck, and Michael~I Jordan.
\newblock Is {Q}-learning provably efficient?
\newblock In \emph{Advances in Neural Information Processing Systems}, pages
  4863--4873, 2018.

\bibitem[Jin et~al.(2019)Jin, Yang, Wang, and Jordan]{jin2019provably}
Chi Jin, Zhuoran Yang, Zhaoran Wang, and Michael~I Jordan.
\newblock Provably efficient reinforcement learning with linear function
  approximation.
\newblock \emph{arXiv preprint arXiv:1907.05388}, 2019.

\bibitem[Kakade(2003)]{kakade2003sample}
Sham~M Kakade.
\newblock \emph{On the sample complexity of reinforcement learning}.
\newblock PhD thesis, University of London London, England, 2003.

\bibitem[Kearns and Singh(1999)]{kearns1999finite}
Michael~J Kearns and Satinder~P Singh.
\newblock Finite-sample convergence rates for {Q}-learning and indirect
  algorithms.
\newblock In \emph{Advances in neural information processing systems}, pages
  996--1002, 1999.

\bibitem[Krishnamurthy et~al.(2016)Krishnamurthy, Agarwal, and
  Langford]{krishnamurthy2016pac}
Akshay Krishnamurthy, Alekh Agarwal, and John Langford.
\newblock {PAC} reinforcement learning with rich observations.
\newblock In \emph{Advances in Neural Information Processing Systems}, pages
  1840--1848, 2016.

\bibitem[Krishnan et~al.(2018)Krishnan, Yang, Goldberg, Hellerstein, and
  Stoica]{krishnan2018learning}
Sanjay Krishnan, Zongheng Yang, Ken Goldberg, Joseph Hellerstein, and Ion
  Stoica.
\newblock Learning to optimize join queries with deep reinforcement learning.
\newblock \emph{arXiv preprint arXiv:1808.03196}, 2018.

\bibitem[Lei et~al.(2012)Lei, Nahum-Shani, Lynch, Oslin, and
  Murphy]{lei2012smart}
Huitan Lei, Inbal Nahum-Shani, Kevin Lynch, David Oslin, and Susan~A Murphy.
\newblock A" smart" design for building individualized treatment sequences.
\newblock \emph{Annual review of clinical psychology}, 8:\penalty0 21--48,
  2012.

\bibitem[Li et~al.(2020)Li, Wei, Chi, Gu, and Chen]{li2020breaking}
Gen Li, Yuing Wei, Yuejie Chi, Yuantao Gu, and Yuxin Chen.
\newblock Breaking the sample size barrier in model-based reinforcement
  learning with a generative model.
\newblock \emph{arXiv preprint arXiv:2005.12900}, 2020.

\bibitem[Li et~al.(2011)Li, Littman, Walsh, and Strehl]{li2011knows}
Lihong Li, Michael~L Littman, Thomas~J Walsh, and Alexander~L Strehl.
\newblock Knows what it knows: a framework for self-aware learning.
\newblock \emph{Machine learning}, 82\penalty0 (3):\penalty0 399--443, 2011.

\bibitem[Mahmud et~al.(2018)Mahmud, Kaiser, Hussain, and
  Vassanelli]{mahmud2018applications}
Mufti Mahmud, Mohammed~Shamim Kaiser, Amir Hussain, and Stefano Vassanelli.
\newblock Applications of deep learning and reinforcement learning to
  biological data.
\newblock \emph{IEEE transactions on neural networks and learning systems},
  29\penalty0 (6):\penalty0 2063--2079, 2018.

\bibitem[Mirhoseini et~al.(2017)Mirhoseini, Pham, Le, Steiner, Larsen, Zhou,
  Kumar, Norouzi, Bengio, and Dean]{mirhoseini2017device}
Azalia Mirhoseini, Hieu Pham, Quoc~V Le, Benoit Steiner, Rasmus Larsen, Yuefeng
  Zhou, Naveen Kumar, Mohammad Norouzi, Samy Bengio, and Jeff Dean.
\newblock Device placement optimization with reinforcement learning.
\newblock \emph{arXiv preprint arXiv:1706.04972}, 2017.

\bibitem[Osband and Roy(2016)]{osband2016on}
Ian Osband and Benjamin~Van Roy.
\newblock On lower bounds for regret in reinforcement learning.
\newblock \emph{ArXiv}, abs/1608.02732, 2016.

\bibitem[Sidford et~al.(2018{\natexlab{a}})Sidford, Wang, Wu, Yang, and
  Ye]{sidford2018near}
Aaron Sidford, Mengdi Wang, Xian Wu, Lin Yang, and Yinyu Ye.
\newblock Near-optimal time and sample complexities for solving {Markov}
  decision processes with a generative model.
\newblock In \emph{Advances in Neural Information Processing Systems}, pages
  5186--5196, 2018{\natexlab{a}}.

\bibitem[Sidford et~al.(2018{\natexlab{b}})Sidford, Wang, Wu, and
  Ye]{sidford2018variance}
Aaron Sidford, Mengdi Wang, Xian Wu, and Yinyu Ye.
\newblock Variance reduced value iteration and faster algorithms for solving
  markov decision processes.
\newblock In \emph{Proceedings of the Twenty-Ninth Annual ACM-SIAM Symposium on
  Discrete Algorithms}, pages 770--787. Society for Industrial and Applied
  Mathematics, 2018{\natexlab{b}}.

\bibitem[Simchowitz and Jamieson(2019)]{max2019nonasymptotic}
Max Simchowitz and Kevin~G Jamieson.
\newblock Non-asymptotic gap-dependent regret bounds for tabular mdps.
\newblock In \emph{Advances in Neural Information Processing Systems}, pages
  1151--1160, 2019.

\bibitem[Singh and Yee(1994)]{singh1994upper}
Satinder~P Singh and Richard~C Yee.
\newblock An upper bound on the loss from approximate optimal-value functions.
\newblock \emph{Machine Learning}, 16\penalty0 (3):\penalty0 227--233, 1994.

\bibitem[Strehl et~al.(2006)Strehl, Li, Wiewiora, Langford, and
  Littman]{strehl2006pac}
Alexander~L Strehl, Lihong Li, Eric Wiewiora, John Langford, and Michael~L
  Littman.
\newblock Pac model-free reinforcement learning.
\newblock In \emph{Proceedings of the 23rd international conference on Machine
  learning}, pages 881--888. ACM, 2006.

\bibitem[Sun et~al.(2018)Sun, Jiang, Krishnamurthy, Agarwal, and
  Langford]{sun2018model}
Wen Sun, Nan Jiang, Akshay Krishnamurthy, Alekh Agarwal, and John Langford.
\newblock Model-based {RL} in contextual decision processes: {PAC} bounds and
  exponential improvements over model-free approaches.
\newblock \emph{arXiv preprint arXiv:1811.08540}, 2018.

\bibitem[Wang et~al.(2020{\natexlab{a}})Wang, Du, Yang, and
  Kakade]{wang2020long}
Ruosong Wang, Simon~S Du, Lin Yang, and Sham Kakade.
\newblock Is long horizon rl more difficult than short horizon rl?
\newblock \emph{Advances in Neural Information Processing Systems}, 33,
  2020{\natexlab{a}}.

\bibitem[Wang et~al.(2020{\natexlab{b}})Wang, Salakhutdinov, and
  Yang]{wang2020provably}
Ruosong Wang, Ruslan Salakhutdinov, and Lin~F Yang.
\newblock Provably efficient reinforcement learning with general value function
  approximation.
\newblock \emph{arXiv preprint arXiv:2005.10804}, 2020{\natexlab{b}}.

\bibitem[Wang et~al.(2019)Wang, Wang, Du, and Krishnamurthy]{wang2019optimism}
Yining Wang, Ruosong Wang, Simon~S Du, and Akshay Krishnamurthy.
\newblock Optimism in reinforcement learning with generalized linear function
  approximation.
\newblock \emph{arXiv preprint arXiv:1912.04136}, 2019.

\bibitem[Wen and Roy(2013)]{wen2013efficient}
Zheng Wen and Benjamin~Van Roy.
\newblock Efficient exploration and value function generalization in
  deterministic systems.
\newblock In \emph{Proceedings of the 26th International Conference on Neural
  Information Processing Systems - Volume 2}, NIPS’13, page 3021–3029, Red
  Hook, NY, USA, 2013. Curran Associates Inc.

\bibitem[Yang and Wang(2019{\natexlab{a}})]{yang2019sample}
Lin Yang and Mengdi Wang.
\newblock Sample-optimal parametric q-learning using linearly additive
  features.
\newblock In \emph{International Conference on Machine Learning}, pages
  6995--7004, 2019{\natexlab{a}}.

\bibitem[Yang and Wang(2019{\natexlab{b}})]{yang2019reinforcement}
Lin~F Yang and Mengdi Wang.
\newblock Reinforcement leaning in feature space: Matrix bandit, kernels, and
  regret bound.
\newblock \emph{arXiv preprint arXiv:1905.10389}, 2019{\natexlab{b}}.

\bibitem[Zanette and Brunskill(2019)]{zanette2019tighter}
Andrea Zanette and Emma Brunskill.
\newblock Tighter problem-dependent regret bounds in reinforcement learning
  without domain knowledge using value function bounds.
\newblock In \emph{International Conference on Machine Learning}, pages
  7304--7312, 2019.

\bibitem[Zanette et~al.(2019{\natexlab{a}})Zanette, Brandfonbrener, Pirotta,
  and Lazaric]{zanette2019frequentist}
Andrea Zanette, David Brandfonbrener, Matteo Pirotta, and Alessandro Lazaric.
\newblock Frequentist regret bounds for randomized least-squares value
  iteration.
\newblock \emph{arXiv preprint arXiv:1911.00567}, 2019{\natexlab{a}}.

\bibitem[Zanette et~al.(2019{\natexlab{b}})Zanette, Kochenderfer, and
  Brunskill]{zanette2019almost}
Andrea Zanette, Mykel~J Kochenderfer, and Emma Brunskill.
\newblock Almost horizon-free structure-aware best policy identification with a
  generative model.
\newblock In \emph{Advances in Neural Information Processing Systems}, pages
  5626--5635, 2019{\natexlab{b}}.

\bibitem[Zanette et~al.(2020)Zanette, Lazaric, Kochenderfer, and
  Brunskill]{zanette2020learning}
Andrea Zanette, Alessandro Lazaric, Mykel Kochenderfer, and Emma Brunskill.
\newblock Learning near optimal policies with low inherent {Bellman} error.
\newblock page arXiv preprint arXiv:2003.00153, 2020.

\bibitem[Zhang et~al.(2020{\natexlab{a}})Zhang, Ji, and
  Du]{zhang2020reinforcement}
Zihan Zhang, Xiangyang Ji, and Simon~S Du.
\newblock Is reinforcement learning more difficult than bandits? a near-optimal
  algorithm escaping the curse of horizon.
\newblock \emph{arXiv preprint arXiv:2009.13503}, 2020{\natexlab{a}}.

\bibitem[Zhang et~al.(2020{\natexlab{b}})Zhang, Zhou, and Ji]{zhang2020optimal}
Zihan Zhang, Yuan Zhou, and Xiangyang Ji.
\newblock Almost optimal model-free reinforcement learning via
  reference-advantage decomposition, 2020{\natexlab{b}}.

\bibitem[Zhao et~al.(2018)Zhao, Xia, Zhang, Ding, Yin, and Tang]{zhao2018deep}
Xiangyu Zhao, Long Xia, Liang Zhang, Zhuoye Ding, Dawei Yin, and Jiliang Tang.
\newblock Deep reinforcement learning for page-wise recommendations.
\newblock In \emph{Proceedings of the 12th ACM Conference on Recommender
  Systems}, pages 95--103, 2018.

\bibitem[Zheng et~al.(2018)Zheng, Zhang, Zheng, Xiang, Yuan, Xie, and
  Li]{zheng2018drn}
Guanjie Zheng, Fuzheng Zhang, Zihan Zheng, Yang Xiang, Nicholas~Jing Yuan, Xing
  Xie, and Zhenhui Li.
\newblock Drn: A deep reinforcement learning framework for news recommendation.
\newblock In \emph{Proceedings of the 2018 World Wide Web Conference}, pages
  167--176, 2018.

\end{thebibliography}
\bibliographystyle{plainnat}

\newpage

\section{Appendix}
\section{Upper Bound Proof}

\subsection{Basic properties of the LSVI algorithm}


In this part, we list some of the important lemmas for the LSVI algorithm, most of which are proven in the previous literature \cite{jin2019provably}. These lemmas are very useful for proving the regret bound in our main theorem.

\begin{lemma}{\rm (Lemma B.3, \cite{jin2019provably}).} \label{B.3}
Define $\left[\mathbb{P}_{h} \W V_{h+1}\right](x, a):=\mathbb{E}_{x^{\prime} \sim \mathbb{P}_{h}(\cdot | x, a)} \W V_{h+1}\left(x^{\prime}\right).$ Under the setting of Theorem \ref{thm:main}, let $c_{\beta}$ be the constant in our definition of $\beta$ (i.e., $\beta=$ $\left.c_{\beta} \cdot d H \sqrt{\iota}\right) .$ There exists an absolute constant $C$ that is independent of $c_{\beta}$ such that for any fixed $p \in[0,1]$ if we let $\mathfrak{E}$ be the event that:
\begin{align*}
&\forall(k, h) \in[K] \times[H]:\\
&\left| \left| \sum_{\tau=1}^{k-1} \phi_{h}^{\tau}\left[\W V_{h+1}^{k}\left(x_{h+1}^{\tau}\right)-\mathbb{P}_{h} \W V_{h+1}^{k}\left(x_{h}^{\tau}, a_{h}^{\tau}\right)\right] \right| \right| _{\left(\Lambda_{h}^{k}\right)^{-1}} \\
&\leq C \cdot d H \sqrt{\chi}
\end{align*}
where $\chi=\log \left[2\left(c_{\beta}+1\right) d T / p\right],$ then $\mathbb{P}(\mathfrak{E}) \geq 1-p / 2.$
\end{lemma}

\paragraph{Remark} We use $\mathcal{V}$ to denote the set of all the value functions in the form of  $$V(\cdot)=\min \left\{\max _{a} \mathbf{w}^{\top} \boldsymbol{\phi}(\cdot, a)+\beta \sqrt{\boldsymbol{\phi}(\cdot, a)^{\top} \Lambda^{-1} \boldsymbol{\phi}(\cdot, a)}, H\right\}.$$ Clearly $\mathcal{V}$ includes all possible value function we generate throughout the algorithm. We can construct a $\varepsilon$-covering of $\mathcal{V}$ with respect to the distance dist$\left(V, V^{\prime}\right)=\sup _{x}\left|V(x)-V^{\prime}(x)\right|.$ In addition, we can prove that $\mathcal{N}_{\varepsilon},$ the $\varepsilon$-covering number of $\mathcal{V},$ can be bounded. Thus, by combining decomposition inequality we can derive this lemma.

\begin{lemma}{\rm (Lemma B.4, \cite{jin2019provably})}\label{B.4}
There exists an absolute constant $c_{\beta}$ such that for $\beta=c_{\beta} \cdot d H \sqrt{\iota}$ where $\iota=\log (2 d T / p),$ and for any fixed policy $\pi,$ on the event $\mathfrak{E}$ defined in Lemma \ref{B.3}, we have for all $(x, a, h, k) \in \mathcal{S} \times \mathcal{A} \times[H] \times[K]$
that:
\begin{align*}
\left\langle\phi(x, a), \mathbf{w}_{h}^{k}\right\rangle- Q_{h}^{\pi}(x, a)=&\mathbb{P}_{h}\left(\W V_{h+1}^{k}-V_{h+1}^{\pi}\right)(x, a)\\
&+\Delta_{h}^{k}(x, a)
\end{align*}
for some $\Delta_{h}^{k}(x, a)$ that satisfies $\left|\Delta_{h}^{k}(x, a)\right| \leq \beta \sqrt{\phi(x, a)^{\top}\left(\Lambda_{h}^{k}\right)^{-1} \phi(x, a)}$
\end{lemma}

\paragraph{Remark.}
It is noted that although the definition of $\W \delta$ is not the same as it in \cite{jin2019provably}, the proof of Lemma \ref{B.4} still holds. Then the following two lemmas can be easily derived by Lemma \ref{B.4} and induction.

\begin{lemma}\label{B.5}{\rm (Lemma B.5 (UCB), \cite{jin2019provably})}. 
On the event 
$\mathfrak{E}$ defined in Lemma \ref{B.3}. we have $$\W Q_{h}^{k}(x, a) \geq Q_{h}^{\star}(x, a)$$ for all $(x, a, h, k) \in \mathcal{S} \times \mathcal{A} \times[H] \times[K]$
\end{lemma}

For any $(h, k) \in [H] \times [k],$ let
$\widetilde \delta_h^{k} = \widetilde V_h^{\tilde{k}}(x_h^{k}) - V_h^{\pi_{\tilde{k}}}(x_h^k)$
denote the errors of the estimated $\W V_h^{\tilde{k}}$ relative to $V_h^{\pi_{\tilde{k}}}$.
\begin{lemma}\label{B.6}{\rm (Lemma B.6 (Recursive Lemma), \cite{jin2019provably})} 
Let $\W \zeta_{h+1}^{k}=\mathbb{E}\left[\W \delta_{h+1}^{k} | x_{h}^{k}, a_{h}^{k}\right]-\W \delta_{h+1}^{k} .$ Then
on the event $\mathfrak{E}$ defined in Lemma \ref{B.3}, we have the following: for any $(k, h) \in[K] \times[H]$
$$
\W \delta_{h}^{k} \leq \W \delta_{h+1}^{k}+\W \zeta_{h+1}^{k}+2 \beta \sqrt{\left(\phi_{h}^{k}\right)^{\top}\left(\Lambda_{h}^{\tilde{k}}\right)^{-1} \phi_{h}^{k}}
$$
\end{lemma}

\begin{lemma}{\rm (Lemma D.2, \cite{jin2019provably})} \label{D.2}
Let $\left\{\phi_{t}\right\}_{t \geq 0}$ be a bounded sequence in $\mathbb{R}^{d}$ satisfying $\sup _{t>0}\left\|\phi_{t}\right\| \leq 1 .$ Let $\Lambda_{0} \in \mathbb{R}^{d \times d}$ be a positive definite matrix. For any $t \geq 0,$ we define $\Lambda_{t}=\Lambda_{0}+\sum_{j=1}^{t} \phi_{j}^{\top} \phi_{j} .$ Then, if the smallest
eigenvalue of $\Lambda_{0}$ satisfies $\lambda_{\min }\left(\Lambda_{0}\right) \geq 1,$ we have
$$
\log \left[\frac{\operatorname{det}\left(\Lambda_{t}\right)}{\operatorname{det}\left(\Lambda_{0}\right)}\right] \leq \sum_{j=1}^{t} \boldsymbol{\phi}_{j}^{\top} \Lambda_{j-1}^{-1} \boldsymbol{\phi}_{j} \leq 2 \log \left[\frac{\operatorname{det}\left(\Lambda_{t}\right)}{\operatorname{det}\left(\Lambda_{0}\right)}\right]
$$
\end{lemma}

\subsection{Decomposing and proving the regret bound}

In this section, we show the decomposition of the regret bound via the following lemma:

\begin{lemma} \label{A.3}
Let  $\widetilde \delta_h^k$ be defined the same as that in the start of this part, then the following bound for the regret holds:
$$
    {\rm Regret}(K) \le \sum_{k = 1}^K \widetilde \delta_1^k
$$
\end{lemma}

The proof of this lemma is straightforward: if we notice that the value function computed by the algorithm always estimates more than the true value, so the following equation holds:  
\begin{equation}
\begin{split}
{\rm Regret}(K) =& \sum_{k = 1} ^{K} [V_1^* (x_1^k) - V_1^{\pi_k}(x_1^k)]\\
\le&\sum_{k = 1}^{K} [\widetilde V_1^{\tilde{k}}(x_1^k) - V_1^{\pi_{\tilde{k}}}(x_1^k)]
\end{split}
\end{equation}

\subsection{Proof of the main theorem: the regret bound}

In this part, we will prove the regret bound of the main theorem stated in the section \ref{sec:algoIntro}.

Firstly, conditioning on the event $\mathfrak{E}$ defined in Lemma \ref{B.4}, we have:

\begin{equation*}
\begin{split}
\left[\W {Q}_{h}^{\tilde{k}}-Q_{h}^{\pi_{\tilde{k}}}\right](x, a)=&{\rm{\mathbb{P}}}_{h}\left(V_{h+1}^{\tilde{k}}-V_{h+1}^{\pi_{\tilde{k}}}\right)\left(x, a\right)+\Delta_{h}^{\tilde{k}}(x, a)\\
\end{split}
\end{equation*}




Our update rule implies $\pi_{\tilde{k}} = \pi_k, \W {Q}_{h}^{\tilde{k}} = {Q}_{h}^{\tilde{k}},$ so subtracting the previous equations, we have

\begin{align*}
\left(\W {Q}_{h}^{\tilde{k}}-Q_{h}^{\pi_{\tilde{k}}}\right)\left(x_{h}^{k}, a_{h}^{k}\right) &\le \Delta_{h}^{\tilde{k}}(x, a)+\W {\delta}_{h}^{k}\\
&\le  \Delta_{h}^{\tilde{k}}(x, a) +
 \W \zeta_h^{k} + 
 \W {\delta}_{h+1}^{k}
\end{align*}

Since $\{ \zeta_h^{k} \}_{k\in [K]}$ 
are bounded martingale difference sequence (adapted to the history up to episode $k-1$), by Azuma-Hoeffding inequality,
we have, with probability at least $1-p/2$
\[
\sum_{k = 1}^{K} \sum_{h = 1}^{H} \W \zeta_h^{k} 
\le H\sqrt{T\iota}.
\]
Let this event be $\mathfrak{E}'$.

Notice our update rule implies $\phi_h^{k} (\Lambda_h^{\tilde{k}})^{-1} \phi_h^{k} \le 2 \phi_h^{k} (\Lambda_h^{k})^{-1} \phi_h^{k} $ , now we can use Lemma \ref{B.3}, \ref{B.6}, \ref{A.3} to do the recursion for the regret bound. In summary, conditioning on $\mathfrak{E}$ and $\mathfrak{E}'$, the following inequalities hold:

\begin{equation*}
\begin{split}
{\rm Regret}(K) \le & \sum_{k = 1}^K \widetilde \delta_1^k \\
\le & \sum_{k = 1}^{K} \sum_{h = 1}^{H} \W \zeta_h^{k} 
+ 3\beta \sum_{k = 1}^{K} \sum_{h = 1}^{H} \sqrt{\phi_h^{k} (\Lambda_h^{\tilde{k}})^{-1} \phi_h^{k}}\\
\le& H\sqrt{T\iota} + 6H\sqrt{dK\iota}\\
\end{split}
\end{equation*}

The 
second part holds from the potential lemma and the Cauchy-Schwarz inequality.

\subsection{Analysis for switching cost}\label{det lemma}
%
\begin{proof}[Proof of Lemma~\ref{detlemma_2}]


Assume $\phi_{\tau}^{d} = (\alpha_1, \alpha_2, \cdots, \alpha_d)^{\top}.$ We know that $\sum_{i = 1}^d \alpha_i^2 \le 1,$ thus $|\alpha_i| \le 1$ for all $i \in [d].$ Since $\phi_{\tau}^{d}(\phi_{\tau}^{d})^{\top} = (\alpha_i \alpha_j)_{ij},$ the absolute value of each component of $\phi_{\tau}^{d}(\phi_{\tau}^{d})^{\top}$ is no more than 1, and then the absolute value of each component of $A_d$ is no more than $K + \lambda.$

Use $\widetilde A_d = (a_{ij})$ to denote a $d$-dimensional matrix satisfying the feature above, $i.e.,$ $|a_{ij}| \le K + \lambda.$ Clearly, if the $1_{st}$ row and $j_{th}$ column are deleted, the rest $(d - 1)$-dimensional matrix is $\widetilde A_{d - 1}.$ Thus

\begin{equation*}
\begin{split}
|\det{A_d}| \le& \sum_{j = 1}^d |a_{1j}| \cdot |\det{A_{d - 1}}|\\
\le& d \cdot (K + \lambda) \cdot |\det{A_{d - 1}}|\\
\le& d^d (K + \lambda)^d\\
\end{split}
\end{equation*}

Hence, $\log \det A_d = d\log d + d \log(K + \lambda) = O(d\log K).$
\end{proof}

To prove Lemma~\ref{detlemma}, we will use the following two linear algebraic facts.
\begin{fact}[Woodbury matrix identity]
For any PSD matrices $A, \Delta\in \mathbb{R}^{d\times d}$, suppose $A$ is invertible, then we have
\begin{align*}
&(A+\Delta)^{-1}\\
=& A^{-1} - A^{-1}\Delta^{1/2}(I + \Delta^{1/2}A^{-1}\Delta^{1/2})^{-1}\Delta^{1/2}A^{-1}.
\end{align*}
\end{fact}
\begin{fact}[Matrix determinant lemma]
For any PSD matrices $A, \Delta\in\mathbb{R}^{d\times d}$, suppose $A$ is invertible, then
\[
\det(A+\Delta) = \det(I + \Delta^{1/2} A^{-1} \Delta^{1/2})\cdot\det(A).
\]
\end{fact}
\begin{proof}[Proof of Lemma~\ref{detlemma}]
By the matrix determinant lemma, we only need to show that 
\[
 \lambda_{\max}(I + \Delta^{1/2} A^{-1} \Delta^{1/2})\ge 2.
\]
Since $A^{-1}\not\preccurlyeq 2B^{-1}$, 
it must be the case that, for some $x$ with $\|x\|_2=1$, and 
\[
x^\top (A^{-1} - 2B^{-1}) x\ge 0.
\]
Denote $\Delta = B- A \succeq 0$.
By Woodbury identity, we have
\begin{align*}
&x^\top (A^{-1} - 2B^{-1})x  \\
=&x^\top (2A^{-1}\Delta^{1/2}(I + \Delta^{1/2}A^{-1}\Delta^{1/2})^{-1}\Delta^{1/2}A^{-1} - A^{-1})x\\  
\ge& 0.
\end{align*}
Let $y = A^{-1/2}x$, we then have,
\begin{align*}
&2y^{\top} A^{-1/2}\Delta^{1/2}(I + \Delta^{1/2}A^{-1}\Delta^{1/2})^{-1}\Delta^{1/2}A^{-1/2} y\\
\ge& \|y\|_2^2.
\end{align*}
Hence, 
\[
\lambda_{\max}(A^{-1/2}\Delta^{1/2}(I + \Delta^{1/2}A^{-1}\Delta^{1/2})^{-1}\Delta^{1/2}A^{-1/2} )
\ge 1/2.
\]
Let us denote $M=A^{-1/2}\Delta^{1/2}$, we have
\[
\lambda_{\max}(M(I + M^\top M)^{-1}M^\top)\ge 1/2.
\]
Let $M = U\Sigma V^\top$ be the SVD decomposition of $M$, where $U$ and $V$ are orthonormal and $\Sigma$ is diagonal.
Then we have,
\begin{align*}
&M(I + M^\top M)^{-1}M^\top\\
= &U\Sigma V^\top(I + V\Sigma^2 V^\top )^{-1} V \Sigma U^\top \\
= &U\Sigma(I + \Sigma^2)^{-1}\Sigma U^\top.
\end{align*}
Note that $\Sigma=\mathrm{diag}(\sigma_1,\sigma_2,
\ldots, \sigma_d)$ is diagonal, we have
\[
\max_i \frac{\sigma_i^2}{1+\sigma_i^2}\ge 1/2
\Rightarrow
\max_{i} \sigma_i^2 \ge 1.
\]
We additionally rewrite $I + \Delta^{1/2} A^{-1} \Delta^{1/2}$ as
\[
I + \Delta^{1/2} A^{-1} \Delta^{1/2} 
= I + M^{\top} M
= I + V\Sigma^2 V^\top.
\]
Thus we have
\[
\lambda_{\max}(I + \Delta^{1/2} A^{-1} \Delta^{1/2})\ge 2
\]
as desired.
\end{proof}


Combining the lemmas above, we are now ready to prove the switching cost bound in the theorem \ref{thm:main}.

\begin{proof}[Proof of Theorem~\ref{thm:main}]
Let $\{ k_1, k_2, \cdots, k_{N_{\rm switch}^{\rm gl}} \}$ denote the $\tilde{k}$ picked by the algorithm. From Lemma \ref{detlemma} we know $\det \Lambda_{k_{i+1}} \ge 2\det \Lambda_{k_{i}} \ge 2^{k_{i+1}}\det \Lambda_0.$ Hence, by combining Lemma \ref{detlemma_2} we have$N_{\rm switch}^{\rm gl} \le c_0 \log \det \Lambda_K = O(d\log K).$
\end{proof}

\section{Analysis for Lower Bound}

We will construct a set of linear MDPs with dimension $d = 4d_0$ and number of steps $H = 2H_0$ to prove theorem \ref{thm:lowerBoundThm}. These MDPs will have different action when facing different states, and we claim that an MDP with the number of action $|\mathcal{A}|$ and the number of steps $H$ can be transformed to an MDP with the number of action 2 and the number of steps $H\log|\mathcal{A}|.$

In fact, we can just add a binary tree  with depth $\log|\mathcal{A}|$ before each transformation, and each intermediate state has two possible action. Then the different leaves of the tree denote for the different chosen of action. Clearly the switching cost will not decrease during this operation.


\subsection{Construction}

For fixed $d_0$, let $e_i$ denote the vector $(0, 0, \cdots, 1, \cdots, 0)$  in $4d_0$-dimensional space, whose $i_{th}$ component is 1 and others are 0s. We first construct a set of linear MDPs $M_{*}$ with dimension $4d_0$ as follows:

The state space $\mathcal{S}$ is partitioned into 4 components, which we define as follows:
\begin{equation*}
\begin{split}
{S_0} = & \{ s_{ h,i} | h \in [H], i \in [d_0]\}\\
{U} = & \{ u \}\\
{V} = & \{ v \}\\
{W} = & \{ w \}\\
\end{split}
\end{equation*}
So we have $\mathcal{S} = {S_0} \cup  {U} \cup  {V} \cup  {W}$. Intuitively, $S_0$ is the space where the agent usually explores. $u, v, w$ are three auxiliary states for ``hiding'' the rewards and normalizing the paths taken by the agent. We take $u$ as the initial state for each episode.

Now we consider the action space as follows:
\begin{equation*}
\begin{split}
{A} =& \{ a_{j} | j \in [d_0] \}\\
{\widetilde A} =& \{ \W a \}\\
\end{split}
\end{equation*}
The agent can take any action in $A$ at state $u,$ while there is only one feasible action $\W a$ for all other states. We use $\mathcal{A} = A \cup \W A$ to denote the whole action space and clearly the maximum number of feasible action in a particular state is $d_0.$

The last information the algorithm knows before exploration is the feature vectors of each state-action pair:
\begin{equation*}
\begin{split}
\phi(u, a_j) = & e_j\\
\phi(s_{h, i}, \W a) = & e_{2d_0 + i}\\
\phi(v, \W a) = & e_{3d_0}\\
\phi(w, \W a) = & e_{4d_0}\\
\end{split}
\end{equation*}
We can easily verify that the agents cannot extract any useful information about the special action by these feature vectors. More precisely, these feature vectors are orthonormal vectors, given constant $h.$
\begin{equation*}
\begin{split}
\mu_{2h}(s_{h,i}) = & e_i\\
\mu_{2h+1}(u) = & e_{2 d_0 + i_h}\\
\mu_{h}(v) = & e_{3d_0},\;\; h\not = 2h_*+1\\
\mu_{2h_*+1}(v) = & e_{j_{h_*}d_0 + i_{h_*}}\\
\mu_{2h}(w) = & e_{4d_0}\\
\mu_{2h+1}(w) = & \sum_{i \in [d_0]} \sum_{j = 0, 1}e_{jd_0 + i} - e_{j_h d_0 + i_h} + e_{4d_0}\\
\end{split}
\end{equation*}
where $h_* \sim \rm{Unif}([H_0]),$ $i_h \sim \rm{Unif}[d_0]$ for all $h \in [h_*]$ and other vectors are all $(0, 0, \cdots, 0).$ We can easily find that $v$ and $w$ are two sinks. As above shows, the agent starts at $u$ and the action $a_j$ leads to state $s_{1, j}.$ If $j = i_h,$ then the agent comes back to $u$ and then selects the next action, else the agent goes to $w$ and stays in $w$ forever. In other words, the agent will finally goes to sink $v$ at step $2h_* + 1$ along with the correct path $i_h,$ or it will go to $w$ if taking any wrong action.
We illustrate the construction in Figure 2. 
Note that denote state $u$ as $u_1, u_2, \ldots, $ for the arrival of the $H$-th step at $u$. 
We do the same for $v$ and $w$.

The reward function is quite simple: the agent gets reward 1 only at state $v:$
\begin{equation*}
\begin{split}
\theta_h = & e_{3d_0}\\
\end{split}
\end{equation*}


\begin{figure}
    \centering
    
    \label{fig:MDPfig}
    \tikzset{every picture/.style={line width=0.75pt}} 

\begin{tikzpicture}[x=0.75pt,y=0.75pt,yscale=-0.8,xscale=0.8]

\draw  [color={rgb, 255:red, 65; green, 117; blue, 5 }  ,draw opacity=1 ] (330,215) .. controls (330,206.72) and (336.72,200) .. (345,200) .. controls (353.28,200) and (360,206.72) .. (360,215) .. controls (360,223.28) and (353.28,230) .. (345,230) .. controls (336.72,230) and (330,223.28) .. (330,215) -- cycle ;
\draw  [color={rgb, 255:red, 65; green, 117; blue, 5 }  ,draw opacity=1 ] (210,215) .. controls (210,206.72) and (216.72,200) .. (225,200) .. controls (233.28,200) and (240,206.72) .. (240,215) .. controls (240,223.28) and (233.28,230) .. (225,230) .. controls (216.72,230) and (210,223.28) .. (210,215) -- cycle ;
\draw  [color={rgb, 255:red, 65; green, 117; blue, 5 }  ,draw opacity=1 ] (90,215) .. controls (90,206.72) and (96.72,200) .. (105,200) .. controls (113.28,200) and (120,206.72) .. (120,215) .. controls (120,223.28) and (113.28,230) .. (105,230) .. controls (96.72,230) and (90,223.28) .. (90,215) -- cycle ;
\draw  [color={rgb, 255:red, 65; green, 117; blue, 5 }  ,draw opacity=1 ] (450,215) .. controls (450,206.72) and (456.72,200) .. (465,200) .. controls (473.28,200) and (480,206.72) .. (480,215) .. controls (480,223.28) and (473.28,230) .. (465,230) .. controls (456.72,230) and (450,223.28) .. (450,215) -- cycle ;
\draw  [color={rgb, 255:red, 65; green, 117; blue, 5 }  ,draw opacity=1 ] (330,125) .. controls (330,116.72) and (336.72,110) .. (345,110) .. controls (353.28,110) and (360,116.72) .. (360,125) .. controls (360,133.28) and (353.28,140) .. (345,140) .. controls (336.72,140) and (330,133.28) .. (330,125) -- cycle ;
\draw  [color={rgb, 255:red, 65; green, 117; blue, 5 }  ,draw opacity=1 ] (210,125) .. controls (210,116.72) and (216.72,110) .. (225,110) .. controls (233.28,110) and (240,116.72) .. (240,125) .. controls (240,133.28) and (233.28,140) .. (225,140) .. controls (216.72,140) and (210,133.28) .. (210,125) -- cycle ;
\draw  [color={rgb, 255:red, 65; green, 117; blue, 5 }  ,draw opacity=1 ] (90,125) .. controls (90,116.72) and (96.72,110) .. (105,110) .. controls (113.28,110) and (120,116.72) .. (120,125) .. controls (120,133.28) and (113.28,140) .. (105,140) .. controls (96.72,140) and (90,133.28) .. (90,125) -- cycle ;
\draw  [color={rgb, 255:red, 65; green, 117; blue, 5 }  ,draw opacity=1 ] (450,125) .. controls (450,116.72) and (456.72,110) .. (465,110) .. controls (473.28,110) and (480,116.72) .. (480,125) .. controls (480,133.28) and (473.28,140) .. (465,140) .. controls (456.72,140) and (450,133.28) .. (450,125) -- cycle ;
\draw  [color={rgb, 255:red, 65; green, 117; blue, 5 }  ,draw opacity=1 ] (330,305) .. controls (330,296.72) and (336.72,290) .. (345,290) .. controls (353.28,290) and (360,296.72) .. (360,305) .. controls (360,313.28) and (353.28,320) .. (345,320) .. controls (336.72,320) and (330,313.28) .. (330,305) -- cycle ;
\draw  [color={rgb, 255:red, 65; green, 117; blue, 5 }  ,draw opacity=1 ] (210,305) .. controls (210,296.72) and (216.72,290) .. (225,290) .. controls (233.28,290) and (240,296.72) .. (240,305) .. controls (240,313.28) and (233.28,320) .. (225,320) .. controls (216.72,320) and (210,313.28) .. (210,305) -- cycle ;
\draw  [color={rgb, 255:red, 65; green, 117; blue, 5 }  ,draw opacity=1 ] (90,305) .. controls (90,296.72) and (96.72,290) .. (105,290) .. controls (113.28,290) and (120,296.72) .. (120,305) .. controls (120,313.28) and (113.28,320) .. (105,320) .. controls (96.72,320) and (90,313.28) .. (90,305) -- cycle ;
\draw  [color={rgb, 255:red, 65; green, 117; blue, 5 }  ,draw opacity=1 ] (450,305) .. controls (450,296.72) and (456.72,290) .. (465,290) .. controls (473.28,290) and (480,296.72) .. (480,305) .. controls (480,313.28) and (473.28,320) .. (465,320) .. controls (456.72,320) and (450,313.28) .. (450,305) -- cycle ;
\draw  [color={rgb, 255:red, 65; green, 117; blue, 5 }  ,draw opacity=1 ] (330,395) .. controls (330,386.72) and (336.72,380) .. (345,380) .. controls (353.28,380) and (360,386.72) .. (360,395) .. controls (360,403.28) and (353.28,410) .. (345,410) .. controls (336.72,410) and (330,403.28) .. (330,395) -- cycle ;
\draw  [color={rgb, 255:red, 65; green, 117; blue, 5 }  ,draw opacity=1 ] (210,395) .. controls (210,386.72) and (216.72,380) .. (225,380) .. controls (233.28,380) and (240,386.72) .. (240,395) .. controls (240,403.28) and (233.28,410) .. (225,410) .. controls (216.72,410) and (210,403.28) .. (210,395) -- cycle ;
\draw  [color={rgb, 255:red, 65; green, 117; blue, 5 }  ,draw opacity=1 ] (90,395) .. controls (90,386.72) and (96.72,380) .. (105,380) .. controls (113.28,380) and (120,386.72) .. (120,395) .. controls (120,403.28) and (113.28,410) .. (105,410) .. controls (96.72,410) and (90,403.28) .. (90,395) -- cycle ;
\draw  [color={rgb, 255:red, 65; green, 117; blue, 5 }  ,draw opacity=1 ] (450,395) .. controls (450,386.72) and (456.72,380) .. (465,380) .. controls (473.28,380) and (480,386.72) .. (480,395) .. controls (480,403.28) and (473.28,410) .. (465,410) .. controls (456.72,410) and (450,403.28) .. (450,395) -- cycle ;
\draw  [color={rgb, 255:red, 144; green, 19; blue, 254 }  ,draw opacity=1 ] (150,260) .. controls (150,251.72) and (156.72,245) .. (165,245) .. controls (173.28,245) and (180,251.72) .. (180,260) .. controls (180,268.28) and (173.28,275) .. (165,275) .. controls (156.72,275) and (150,268.28) .. (150,260) -- cycle ;
\draw  [color={rgb, 255:red, 144; green, 19; blue, 254 }  ,draw opacity=1 ] (150,350) .. controls (150,341.72) and (156.72,335) .. (165,335) .. controls (173.28,335) and (180,341.72) .. (180,350) .. controls (180,358.28) and (173.28,365) .. (165,365) .. controls (156.72,365) and (150,358.28) .. (150,350) -- cycle ;
\draw  [color={rgb, 255:red, 144; green, 19; blue, 254 }  ,draw opacity=1 ] (150,170) .. controls (150,161.72) and (156.72,155) .. (165,155) .. controls (173.28,155) and (180,161.72) .. (180,170) .. controls (180,178.28) and (173.28,185) .. (165,185) .. controls (156.72,185) and (150,178.28) .. (150,170) -- cycle ;
\draw  [color={rgb, 255:red, 144; green, 19; blue, 254 }  ,draw opacity=1 ] (150,80) .. controls (150,71.72) and (156.72,65) .. (165,65) .. controls (173.28,65) and (180,71.72) .. (180,80) .. controls (180,88.28) and (173.28,95) .. (165,95) .. controls (156.72,95) and (150,88.28) .. (150,80) -- cycle ;
\draw  [color={rgb, 255:red, 139; green, 87; blue, 42 }  ,draw opacity=1 ] (270,260) .. controls (270,251.72) and (276.72,245) .. (285,245) .. controls (293.28,245) and (300,251.72) .. (300,260) .. controls (300,268.28) and (293.28,275) .. (285,275) .. controls (276.72,275) and (270,268.28) .. (270,260) -- cycle ;
\draw  [color={rgb, 255:red, 139; green, 87; blue, 42 }  ,draw opacity=1 ] (270,350) .. controls (270,341.72) and (276.72,335) .. (285,335) .. controls (293.28,335) and (300,341.72) .. (300,350) .. controls (300,358.28) and (293.28,365) .. (285,365) .. controls (276.72,365) and (270,358.28) .. (270,350) -- cycle ;
\draw  [color={rgb, 255:red, 139; green, 87; blue, 42 }  ,draw opacity=1 ] (270,170) .. controls (270,161.72) and (276.72,155) .. (285,155) .. controls (293.28,155) and (300,161.72) .. (300,170) .. controls (300,178.28) and (293.28,185) .. (285,185) .. controls (276.72,185) and (270,178.28) .. (270,170) -- cycle ;
\draw  [color={rgb, 255:red, 139; green, 87; blue, 42 }  ,draw opacity=1 ] (270,80) .. controls (270,71.72) and (276.72,65) .. (285,65) .. controls (293.28,65) and (300,71.72) .. (300,80) .. controls (300,88.28) and (293.28,95) .. (285,95) .. controls (276.72,95) and (270,88.28) .. (270,80) -- cycle ;
\draw  [color={rgb, 255:red, 208; green, 2; blue, 27 }  ,draw opacity=1 ] (390,260) .. controls (390,251.72) and (396.72,245) .. (405,245) .. controls (413.28,245) and (420,251.72) .. (420,260) .. controls (420,268.28) and (413.28,275) .. (405,275) .. controls (396.72,275) and (390,268.28) .. (390,260) -- cycle ;
\draw  [color={rgb, 255:red, 208; green, 2; blue, 27 }  ,draw opacity=1 ] (390,350) .. controls (390,341.72) and (396.72,335) .. (405,335) .. controls (413.28,335) and (420,341.72) .. (420,350) .. controls (420,358.28) and (413.28,365) .. (405,365) .. controls (396.72,365) and (390,358.28) .. (390,350) -- cycle ;
\draw  [color={rgb, 255:red, 208; green, 2; blue, 27 }  ,draw opacity=1 ] (390,170) .. controls (390,161.72) and (396.72,155) .. (405,155) .. controls (413.28,155) and (420,161.72) .. (420,170) .. controls (420,178.28) and (413.28,185) .. (405,185) .. controls (396.72,185) and (390,178.28) .. (390,170) -- cycle ;
\draw  [color={rgb, 255:red, 208; green, 2; blue, 27 }  ,draw opacity=1 ] (390,80) .. controls (390,71.72) and (396.72,65) .. (405,65) .. controls (413.28,65) and (420,71.72) .. (420,80) .. controls (420,88.28) and (413.28,95) .. (405,95) .. controls (396.72,95) and (390,88.28) .. (390,80) -- cycle ;
\draw [color={rgb, 255:red, 126; green, 211; blue, 33 }  ,draw opacity=1 ]   (165,80) -- (106.6,123.8) ;
\draw [shift={(105,125)}, rotate = 323.13] [color={rgb, 255:red, 126; green, 211; blue, 33 }  ,draw opacity=1 ][line width=0.75]    (10.93,-3.29) .. controls (6.95,-1.4) and (3.31,-0.3) .. (0,0) .. controls (3.31,0.3) and (6.95,1.4) .. (10.93,3.29)   ;
\draw [color={rgb, 255:red, 126; green, 211; blue, 33 }  ,draw opacity=1 ]   (165,80) -- (223.4,123.8) ;
\draw [shift={(225,125)}, rotate = 216.87] [color={rgb, 255:red, 126; green, 211; blue, 33 }  ,draw opacity=1 ][line width=0.75]    (10.93,-3.29) .. controls (6.95,-1.4) and (3.31,-0.3) .. (0,0) .. controls (3.31,0.3) and (6.95,1.4) .. (10.93,3.29)   ;
\draw [color={rgb, 255:red, 126; green, 211; blue, 33 }  ,draw opacity=1 ]   (165,80) -- (343.06,124.51) ;
\draw [shift={(345,125)}, rotate = 194.04] [color={rgb, 255:red, 126; green, 211; blue, 33 }  ,draw opacity=1 ][line width=0.75]    (10.93,-3.29) .. controls (6.95,-1.4) and (3.31,-0.3) .. (0,0) .. controls (3.31,0.3) and (6.95,1.4) .. (10.93,3.29)   ;
\draw [color={rgb, 255:red, 126; green, 211; blue, 33 }  ,draw opacity=1 ]   (165,80) -- (463.02,124.7) ;
\draw [shift={(465,125)}, rotate = 188.53] [color={rgb, 255:red, 126; green, 211; blue, 33 }  ,draw opacity=1 ][line width=0.75]    (10.93,-3.29) .. controls (6.95,-1.4) and (3.31,-0.3) .. (0,0) .. controls (3.31,0.3) and (6.95,1.4) .. (10.93,3.29)   ;
\draw [color={rgb, 255:red, 126; green, 211; blue, 33 }  ,draw opacity=1 ]   (165,170) -- (106.6,213.8) ;
\draw [shift={(105,215)}, rotate = 323.13] [color={rgb, 255:red, 126; green, 211; blue, 33 }  ,draw opacity=1 ][line width=0.75]    (10.93,-3.29) .. controls (6.95,-1.4) and (3.31,-0.3) .. (0,0) .. controls (3.31,0.3) and (6.95,1.4) .. (10.93,3.29)   ;
\draw [color={rgb, 255:red, 126; green, 211; blue, 33 }  ,draw opacity=1 ]   (165,170) -- (223.4,213.8) ;
\draw [shift={(225,215)}, rotate = 216.87] [color={rgb, 255:red, 126; green, 211; blue, 33 }  ,draw opacity=1 ][line width=0.75]    (10.93,-3.29) .. controls (6.95,-1.4) and (3.31,-0.3) .. (0,0) .. controls (3.31,0.3) and (6.95,1.4) .. (10.93,3.29)   ;
\draw [color={rgb, 255:red, 126; green, 211; blue, 33 }  ,draw opacity=1 ]   (165,170) -- (343.06,214.51) ;
\draw [shift={(345,215)}, rotate = 194.04] [color={rgb, 255:red, 126; green, 211; blue, 33 }  ,draw opacity=1 ][line width=0.75]    (10.93,-3.29) .. controls (6.95,-1.4) and (3.31,-0.3) .. (0,0) .. controls (3.31,0.3) and (6.95,1.4) .. (10.93,3.29)   ;
\draw [color={rgb, 255:red, 126; green, 211; blue, 33 }  ,draw opacity=1 ]   (165,170) -- (463.02,214.7) ;
\draw [shift={(465,215)}, rotate = 188.53] [color={rgb, 255:red, 126; green, 211; blue, 33 }  ,draw opacity=1 ][line width=0.75]    (10.93,-3.29) .. controls (6.95,-1.4) and (3.31,-0.3) .. (0,0) .. controls (3.31,0.3) and (6.95,1.4) .. (10.93,3.29)   ;
\draw [color={rgb, 255:red, 126; green, 211; blue, 33 }  ,draw opacity=1 ]   (165,260) -- (106.6,303.8) ;
\draw [shift={(105,305)}, rotate = 323.13] [color={rgb, 255:red, 126; green, 211; blue, 33 }  ,draw opacity=1 ][line width=0.75]    (10.93,-3.29) .. controls (6.95,-1.4) and (3.31,-0.3) .. (0,0) .. controls (3.31,0.3) and (6.95,1.4) .. (10.93,3.29)   ;
\draw [color={rgb, 255:red, 126; green, 211; blue, 33 }  ,draw opacity=1 ]   (165,260) -- (223.4,303.8) ;
\draw [shift={(225,305)}, rotate = 216.87] [color={rgb, 255:red, 126; green, 211; blue, 33 }  ,draw opacity=1 ][line width=0.75]    (10.93,-3.29) .. controls (6.95,-1.4) and (3.31,-0.3) .. (0,0) .. controls (3.31,0.3) and (6.95,1.4) .. (10.93,3.29)   ;
\draw [color={rgb, 255:red, 126; green, 211; blue, 33 }  ,draw opacity=1 ]   (165,260) -- (343.06,304.51) ;
\draw [shift={(345,305)}, rotate = 194.04] [color={rgb, 255:red, 126; green, 211; blue, 33 }  ,draw opacity=1 ][line width=0.75]    (10.93,-3.29) .. controls (6.95,-1.4) and (3.31,-0.3) .. (0,0) .. controls (3.31,0.3) and (6.95,1.4) .. (10.93,3.29)   ;
\draw [color={rgb, 255:red, 126; green, 211; blue, 33 }  ,draw opacity=1 ]   (165,260) -- (463.02,304.7) ;
\draw [shift={(465,305)}, rotate = 188.53] [color={rgb, 255:red, 126; green, 211; blue, 33 }  ,draw opacity=1 ][line width=0.75]    (10.93,-3.29) .. controls (6.95,-1.4) and (3.31,-0.3) .. (0,0) .. controls (3.31,0.3) and (6.95,1.4) .. (10.93,3.29)   ;
\draw [color={rgb, 255:red, 126; green, 211; blue, 33 }  ,draw opacity=1 ]   (165,350) -- (106.6,393.8) ;
\draw [shift={(105,395)}, rotate = 323.13] [color={rgb, 255:red, 126; green, 211; blue, 33 }  ,draw opacity=1 ][line width=0.75]    (10.93,-3.29) .. controls (6.95,-1.4) and (3.31,-0.3) .. (0,0) .. controls (3.31,0.3) and (6.95,1.4) .. (10.93,3.29)   ;
\draw [color={rgb, 255:red, 126; green, 211; blue, 33 }  ,draw opacity=1 ]   (165,350) -- (223.4,393.8) ;
\draw [shift={(225,395)}, rotate = 216.87] [color={rgb, 255:red, 126; green, 211; blue, 33 }  ,draw opacity=1 ][line width=0.75]    (10.93,-3.29) .. controls (6.95,-1.4) and (3.31,-0.3) .. (0,0) .. controls (3.31,0.3) and (6.95,1.4) .. (10.93,3.29)   ;
\draw [color={rgb, 255:red, 126; green, 211; blue, 33 }  ,draw opacity=1 ]   (165,350) -- (343.06,394.51) ;
\draw [shift={(345,395)}, rotate = 194.04] [color={rgb, 255:red, 126; green, 211; blue, 33 }  ,draw opacity=1 ][line width=0.75]    (10.93,-3.29) .. controls (6.95,-1.4) and (3.31,-0.3) .. (0,0) .. controls (3.31,0.3) and (6.95,1.4) .. (10.93,3.29)   ;
\draw [color={rgb, 255:red, 126; green, 211; blue, 33 }  ,draw opacity=1 ]   (165,350) -- (463.02,394.7) ;
\draw [shift={(465,395)}, rotate = 188.53] [color={rgb, 255:red, 126; green, 211; blue, 33 }  ,draw opacity=1 ][line width=0.75]    (10.93,-3.29) .. controls (6.95,-1.4) and (3.31,-0.3) .. (0,0) .. controls (3.31,0.3) and (6.95,1.4) .. (10.93,3.29)   ;
\draw [color={rgb, 255:red, 139; green, 87; blue, 42 }  ,draw opacity=0.53 ]   (285,80) -- (285,168) ;
\draw [shift={(285,170)}, rotate = 270] [color={rgb, 255:red, 139; green, 87; blue, 42 }  ,draw opacity=0.53 ][line width=0.75]    (10.93,-3.29) .. controls (6.95,-1.4) and (3.31,-0.3) .. (0,0) .. controls (3.31,0.3) and (6.95,1.4) .. (10.93,3.29)   ;
\draw [color={rgb, 255:red, 139; green, 87; blue, 42 }  ,draw opacity=0.53 ]   (285,170) -- (285,258) ;
\draw [shift={(285,260)}, rotate = 270] [color={rgb, 255:red, 139; green, 87; blue, 42 }  ,draw opacity=0.53 ][line width=0.75]    (10.93,-3.29) .. controls (6.95,-1.4) and (3.31,-0.3) .. (0,0) .. controls (3.31,0.3) and (6.95,1.4) .. (10.93,3.29)   ;
\draw [color={rgb, 255:red, 139; green, 87; blue, 42 }  ,draw opacity=0.53 ]   (285,260) -- (285,348) ;
\draw [shift={(285,350)}, rotate = 270] [color={rgb, 255:red, 139; green, 87; blue, 42 }  ,draw opacity=0.53 ][line width=0.75]    (10.93,-3.29) .. controls (6.95,-1.4) and (3.31,-0.3) .. (0,0) .. controls (3.31,0.3) and (6.95,1.4) .. (10.93,3.29)   ;
\draw [color={rgb, 255:red, 208; green, 2; blue, 27 }  ,draw opacity=0.42 ]   (405,80) -- (405,168) ;
\draw [shift={(405,170)}, rotate = 270] [color={rgb, 255:red, 208; green, 2; blue, 27 }  ,draw opacity=0.42 ][line width=0.75]    (10.93,-3.29) .. controls (6.95,-1.4) and (3.31,-0.3) .. (0,0) .. controls (3.31,0.3) and (6.95,1.4) .. (10.93,3.29)   ;
\draw [color={rgb, 255:red, 208; green, 2; blue, 27 }  ,draw opacity=0.42 ]   (405,170) -- (405,258) ;
\draw [shift={(405,260)}, rotate = 270] [color={rgb, 255:red, 208; green, 2; blue, 27 }  ,draw opacity=0.42 ][line width=0.75]    (10.93,-3.29) .. controls (6.95,-1.4) and (3.31,-0.3) .. (0,0) .. controls (3.31,0.3) and (6.95,1.4) .. (10.93,3.29)   ;
\draw [color={rgb, 255:red, 208; green, 2; blue, 27 }  ,draw opacity=0.42 ]   (405,260) -- (405,348) ;
\draw [shift={(405,350)}, rotate = 270] [color={rgb, 255:red, 208; green, 2; blue, 27 }  ,draw opacity=0.42 ][line width=0.75]    (10.93,-3.29) .. controls (6.95,-1.4) and (3.31,-0.3) .. (0,0) .. controls (3.31,0.3) and (6.95,1.4) .. (10.93,3.29)   ;
\draw  [color={rgb, 255:red, 139; green, 87; blue, 42 }  ,draw opacity=1 ] (270,440) .. controls (270,431.72) and (276.72,425) .. (285,425) .. controls (293.28,425) and (300,431.72) .. (300,440) .. controls (300,448.28) and (293.28,455) .. (285,455) .. controls (276.72,455) and (270,448.28) .. (270,440) -- cycle ;
\draw  [color={rgb, 255:red, 208; green, 2; blue, 27 }  ,draw opacity=1 ] (390,440) .. controls (390,431.72) and (396.72,425) .. (405,425) .. controls (413.28,425) and (420,431.72) .. (420,440) .. controls (420,448.28) and (413.28,455) .. (405,455) .. controls (396.72,455) and (390,448.28) .. (390,440) -- cycle ;
\draw [color={rgb, 255:red, 139; green, 87; blue, 42 }  ,draw opacity=0.53 ]   (285,350) -- (285,438) ;
\draw [shift={(285,440)}, rotate = 270] [color={rgb, 255:red, 139; green, 87; blue, 42 }  ,draw opacity=0.53 ][line width=0.75]    (10.93,-3.29) .. controls (6.95,-1.4) and (3.31,-0.3) .. (0,0) .. controls (3.31,0.3) and (6.95,1.4) .. (10.93,3.29)   ;
\draw [color={rgb, 255:red, 208; green, 2; blue, 27 }  ,draw opacity=0.42 ]   (405,350) -- (405,438) ;
\draw [shift={(405,440)}, rotate = 270] [color={rgb, 255:red, 208; green, 2; blue, 27 }  ,draw opacity=0.42 ][line width=0.75]    (10.93,-3.29) .. controls (6.95,-1.4) and (3.31,-0.3) .. (0,0) .. controls (3.31,0.3) and (6.95,1.4) .. (10.93,3.29)   ;
\draw [color={rgb, 255:red, 208; green, 2; blue, 27 }  ,draw opacity=0.42 ]   (465,125) -- (406.6,168.8) ;
\draw [shift={(405,170)}, rotate = 323.13] [color={rgb, 255:red, 208; green, 2; blue, 27 }  ,draw opacity=0.42 ][line width=0.75]    (10.93,-3.29) .. controls (6.95,-1.4) and (3.31,-0.3) .. (0,0) .. controls (3.31,0.3) and (6.95,1.4) .. (10.93,3.29)   ;
\draw [color={rgb, 255:red, 208; green, 2; blue, 27 }  ,draw opacity=0.42 ]   (465,215) -- (406.6,258.8) ;
\draw [shift={(405,260)}, rotate = 323.13] [color={rgb, 255:red, 208; green, 2; blue, 27 }  ,draw opacity=0.42 ][line width=0.75]    (10.93,-3.29) .. controls (6.95,-1.4) and (3.31,-0.3) .. (0,0) .. controls (3.31,0.3) and (6.95,1.4) .. (10.93,3.29)   ;
\draw [color={rgb, 255:red, 208; green, 2; blue, 27 }  ,draw opacity=0.42 ]   (465,305) -- (406.6,348.8) ;
\draw [shift={(405,350)}, rotate = 323.13] [color={rgb, 255:red, 208; green, 2; blue, 27 }  ,draw opacity=0.42 ][line width=0.75]    (10.93,-3.29) .. controls (6.95,-1.4) and (3.31,-0.3) .. (0,0) .. controls (3.31,0.3) and (6.95,1.4) .. (10.93,3.29)   ;
\draw [color={rgb, 255:red, 208; green, 2; blue, 27 }  ,draw opacity=0.42 ]   (465,395) -- (406.6,438.8) ;
\draw [shift={(405,440)}, rotate = 323.13] [color={rgb, 255:red, 208; green, 2; blue, 27 }  ,draw opacity=0.42 ][line width=0.75]    (10.93,-3.29) .. controls (6.95,-1.4) and (3.31,-0.3) .. (0,0) .. controls (3.31,0.3) and (6.95,1.4) .. (10.93,3.29)   ;
\draw [color={rgb, 255:red, 208; green, 2; blue, 27 }  ,draw opacity=0.42 ]   (105,395) -- (403.02,439.7) ;
\draw [shift={(405,440)}, rotate = 188.53] [color={rgb, 255:red, 208; green, 2; blue, 27 }  ,draw opacity=0.42 ][line width=0.75]    (10.93,-3.29) .. controls (6.95,-1.4) and (3.31,-0.3) .. (0,0) .. controls (3.31,0.3) and (6.95,1.4) .. (10.93,3.29)   ;
\draw [color={rgb, 255:red, 208; green, 2; blue, 27 }  ,draw opacity=0.42 ]   (225,395) -- (403.06,439.51) ;
\draw [shift={(405,440)}, rotate = 194.04] [color={rgb, 255:red, 208; green, 2; blue, 27 }  ,draw opacity=0.42 ][line width=0.75]    (10.93,-3.29) .. controls (6.95,-1.4) and (3.31,-0.3) .. (0,0) .. controls (3.31,0.3) and (6.95,1.4) .. (10.93,3.29)   ;
\draw [color={rgb, 255:red, 208; green, 2; blue, 27 }  ,draw opacity=0.42 ]   (345,395) -- (403.4,438.8) ;
\draw [shift={(405,440)}, rotate = 216.87] [color={rgb, 255:red, 208; green, 2; blue, 27 }  ,draw opacity=0.42 ][line width=0.75]    (10.93,-3.29) .. controls (6.95,-1.4) and (3.31,-0.3) .. (0,0) .. controls (3.31,0.3) and (6.95,1.4) .. (10.93,3.29)   ;
\draw [color={rgb, 255:red, 208; green, 2; blue, 27 }  ,draw opacity=0.42 ]   (105,305) -- (403.02,349.7) ;
\draw [shift={(405,350)}, rotate = 188.53] [color={rgb, 255:red, 208; green, 2; blue, 27 }  ,draw opacity=0.42 ][line width=0.75]    (10.93,-3.29) .. controls (6.95,-1.4) and (3.31,-0.3) .. (0,0) .. controls (3.31,0.3) and (6.95,1.4) .. (10.93,3.29)   ;
\draw [color={rgb, 255:red, 208; green, 2; blue, 27 }  ,draw opacity=0.42 ]   (105,215) -- (403.02,259.7) ;
\draw [shift={(405,260)}, rotate = 188.53] [color={rgb, 255:red, 208; green, 2; blue, 27 }  ,draw opacity=0.42 ][line width=0.75]    (10.93,-3.29) .. controls (6.95,-1.4) and (3.31,-0.3) .. (0,0) .. controls (3.31,0.3) and (6.95,1.4) .. (10.93,3.29)   ;
\draw [color={rgb, 255:red, 208; green, 2; blue, 27 }  ,draw opacity=0.42 ]   (225,305) -- (403.06,349.51) ;
\draw [shift={(405,350)}, rotate = 194.04] [color={rgb, 255:red, 208; green, 2; blue, 27 }  ,draw opacity=0.42 ][line width=0.75]    (10.93,-3.29) .. controls (6.95,-1.4) and (3.31,-0.3) .. (0,0) .. controls (3.31,0.3) and (6.95,1.4) .. (10.93,3.29)   ;
\draw [color={rgb, 255:red, 208; green, 2; blue, 27 }  ,draw opacity=0.42 ]   (225,125) -- (403.06,169.51) ;
\draw [shift={(405,170)}, rotate = 194.04] [color={rgb, 255:red, 208; green, 2; blue, 27 }  ,draw opacity=0.42 ][line width=0.75]    (10.93,-3.29) .. controls (6.95,-1.4) and (3.31,-0.3) .. (0,0) .. controls (3.31,0.3) and (6.95,1.4) .. (10.93,3.29)   ;
\draw [color={rgb, 255:red, 208; green, 2; blue, 27 }  ,draw opacity=0.42 ]   (345,215) -- (403.4,258.8) ;
\draw [shift={(405,260)}, rotate = 216.87] [color={rgb, 255:red, 208; green, 2; blue, 27 }  ,draw opacity=0.42 ][line width=0.75]    (10.93,-3.29) .. controls (6.95,-1.4) and (3.31,-0.3) .. (0,0) .. controls (3.31,0.3) and (6.95,1.4) .. (10.93,3.29)   ;
\draw [color={rgb, 255:red, 208; green, 2; blue, 27 }  ,draw opacity=0.42 ]   (345,125) -- (403.4,168.8) ;
\draw [shift={(405,170)}, rotate = 216.87] [color={rgb, 255:red, 208; green, 2; blue, 27 }  ,draw opacity=0.42 ][line width=0.75]    (10.93,-3.29) .. controls (6.95,-1.4) and (3.31,-0.3) .. (0,0) .. controls (3.31,0.3) and (6.95,1.4) .. (10.93,3.29)   ;
\draw [color={rgb, 255:red, 144; green, 19; blue, 254 }  ,draw opacity=0.42 ]   (105,125) -- (163.4,168.8) ;
\draw [shift={(165,170)}, rotate = 216.87] [color={rgb, 255:red, 144; green, 19; blue, 254 }  ,draw opacity=0.42 ][line width=0.75]    (10.93,-3.29) .. controls (6.95,-1.4) and (3.31,-0.3) .. (0,0) .. controls (3.31,0.3) and (6.95,1.4) .. (10.93,3.29)   ;
\draw [color={rgb, 255:red, 139; green, 87; blue, 42 }  ,draw opacity=0.53 ]   (345,305) -- (286.6,348.8) ;
\draw [shift={(285,350)}, rotate = 323.13] [color={rgb, 255:red, 139; green, 87; blue, 42 }  ,draw opacity=0.53 ][line width=0.75]    (10.93,-3.29) .. controls (6.95,-1.4) and (3.31,-0.3) .. (0,0) .. controls (3.31,0.3) and (6.95,1.4) .. (10.93,3.29)   ;
\draw [color={rgb, 255:red, 144; green, 19; blue, 254 }  ,draw opacity=0.42 ]   (225,215) -- (166.6,258.8) ;
\draw [shift={(165,260)}, rotate = 323.13] [color={rgb, 255:red, 144; green, 19; blue, 254 }  ,draw opacity=0.42 ][line width=0.75]    (10.93,-3.29) .. controls (6.95,-1.4) and (3.31,-0.3) .. (0,0) .. controls (3.31,0.3) and (6.95,1.4) .. (10.93,3.29)   ;

\draw (92,120) node [anchor=north west][inner sep=0.75pt]   [align=left] {$\displaystyle s_{1,1}$};
\draw (212,120) node [anchor=north west][inner sep=0.75pt]   [align=left] {$\displaystyle s_{1,2}$};
\draw (332,120) node [anchor=north west][inner sep=0.75pt]   [align=left] {$\displaystyle s_{1,3}$};
\draw (452,120) node [anchor=north west][inner sep=0.75pt]   [align=left] {$\displaystyle s_{1,4}$};
\draw (92,210) node [anchor=north west][inner sep=0.75pt]   [align=left] {$\displaystyle s_{2,1}$};
\draw (212,210) node [anchor=north west][inner sep=0.75pt]   [align=left] {$\displaystyle s_{2,2}$};
\draw (332,210) node [anchor=north west][inner sep=0.75pt]   [align=left] {$\displaystyle s_{2,3}$};
\draw (452,210) node [anchor=north west][inner sep=0.75pt]   [align=left] {$\displaystyle s_{2,4}$};
\draw (92,300) node [anchor=north west][inner sep=0.75pt]   [align=left] {$\displaystyle s_{3,1}$};
\draw (212,300) node [anchor=north west][inner sep=0.75pt]   [align=left] {$\displaystyle s_{3,2}$};
\draw (332,300) node [anchor=north west][inner sep=0.75pt]   [align=left] {$\displaystyle s_{3,3}$};
\draw (452,300) node [anchor=north west][inner sep=0.75pt]   [align=left] {$\displaystyle s_{3,4}$};
\draw (92,390) node [anchor=north west][inner sep=0.75pt]   [align=left] {$\displaystyle s_{4,1}$};
\draw (212,390) node [anchor=north west][inner sep=0.75pt]   [align=left] {$\displaystyle s_{4,2}$};
\draw (332,390) node [anchor=north west][inner sep=0.75pt]   [align=left] {$\displaystyle s_{4,3}$};
\draw (452,390) node [anchor=north west][inner sep=0.75pt]   [align=left] {$\displaystyle s_{4,4}$};
\draw (157,255) node [anchor=north west][inner sep=0.75pt]   [align=left] {$\displaystyle u_{3}$};
\draw (157,345) node [anchor=north west][inner sep=0.75pt]   [align=left] {$\displaystyle u_{4}$};
\draw (157,165) node [anchor=north west][inner sep=0.75pt]   [align=left] {$\displaystyle u_{2}$};
\draw (157,75) node [anchor=north west][inner sep=0.75pt]   [align=left] {$\displaystyle u_{1}$};
\draw (277,255) node [anchor=north west][inner sep=0.75pt]   [align=left] {$\displaystyle v_{3}$};
\draw (277,345) node [anchor=north west][inner sep=0.75pt]   [align=left] {$\displaystyle v_{4}$};
\draw (277,165) node [anchor=north west][inner sep=0.75pt]   [align=left] {$\displaystyle v_{2}$};
\draw (277,75) node [anchor=north west][inner sep=0.75pt]   [align=left] {$\displaystyle v_{1}$};
\draw (397,255) node [anchor=north west][inner sep=0.75pt]   [align=left] {$\displaystyle w_{3}$};
\draw (397,345) node [anchor=north west][inner sep=0.75pt]   [align=left] {$\displaystyle w_{4}$};
\draw (397,165) node [anchor=north west][inner sep=0.75pt]   [align=left] {$\displaystyle w_{2}$};
\draw (397,75) node [anchor=north west][inner sep=0.75pt]   [align=left] {$\displaystyle w_{1}$};
\draw (277,435) node [anchor=north west][inner sep=0.75pt]   [align=left] {$\displaystyle v_{5}$};
\draw (397,435) node [anchor=north west][inner sep=0.75pt]   [align=left] {$\displaystyle w_{5}$};

\end{tikzpicture}

\caption{An illustration of the hard instance constructed above}
\end{figure}
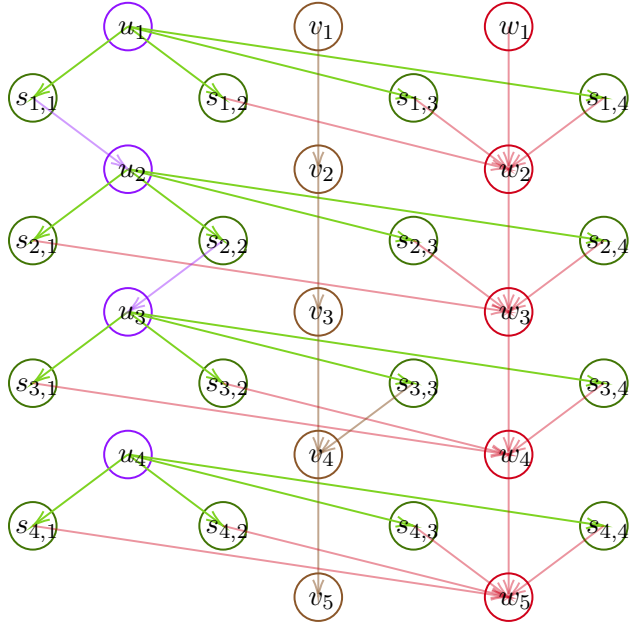

Clearly we have
\begin{equation*}
\begin{split}
\mathbb{E}_{M_{*}}[{V_1^\pi(x_0)}] =& \E_{M_{*}} \left[ \sum_{h = 1}^H \mathbf{1} \{ x_h = v \} \right]\\
\E_{M_{*}}[V_1^*(x_0)] =& \E_{M_{*}}[H - 2h_*] = H_0\\
\end{split}
\end{equation*}

\subsection{Minimax Lower Bound}
\begin{equation*}
\begin{split}
&\sup_{M \in \mathcal{M}} \mathbb{E}_{s_1, \mathcal{M}}{\left[ \sum_{k = 1}^{K} V_1^*(s_1) - V_1^{\pi_k}(s_1) \right ]} \\
\ge& \E_{M_{*}} [\sum_{k = 1}^K V_1^* (x_0) - V_1^{\pi_k} (x_0)]\\
=& KH_0 - \sum_{k = 1}^K \E_{M_{*}}[V_1^{\pi_k}(x_0)]
\end{split}
\end{equation*}
It remains to upper bound $\E_{M_{*}}[V_1^{\pi_k}(x_0)]$ for each $k.$

For all $k \ge 1,$ let 
$$ N_{\rm switch}^k = \sum_{j = 1}^{k - 1} \mathbf{1} \{ \pi_{j} \not = \pi_{j+1} \} $$ 
denote the switching cost at episode $K$. Let
$$S_* :=  \{ s_{1, i_1}, s_{2, i_2}, \cdots, s_{h_*, i_{h_*}} \}$$
be the correct path leading to state $v,$
\begin{equation*}
\begin{split}
S_k := & \{ \tilde s_{h_1, i_1}, \tilde s_{h_2, i_2}, \cdots, \tilde s_{h_{r}, i_{r}} \}\\
\end{split}
\end{equation*}
be the ordered set of the states $s_{h, i}$ that have been reached throughout the execution of the algorithm,
$$ S_{k}^{\tau} := \{ \tilde s _{\tau, i_1}, \tilde s _{\tau, i_2}, \cdots, \tilde s _{{\tau}, i_{r_{\tau}}} \} $$
be the states throughout the exploration of $s_{\tau, i_{\tau}},$ i.e. $S_{k}^{\tau} = S_k \cap \{ s_{\tau, i} | i \in [d_0] \}.$

We begin by observing that $r = \sum_{\tau = 1}^{h_*} r_{\tau} \le N_{\rm switch}^k + H_0 + 1,$ i.e., after changing the policy for $N_{\rm switch}^k$ times, the algorithm can only explore at most $N_{\rm switch}^k$ states in ${S'}$ except for the correct path. In fact, we know that if $x_{2h} \in S' - S_*$ for some $h,$ then $x_{2h+1} = w$
 and so do the rest steps. Thus as long as the algorithm makes a mistake at some step, it can only explore one more state in $S' - S_*.$
 
More precisely, if the algorithm has already known the structure of this MDP, clearly it still need to find the correct path $S_*$ in order to achieve reward 1. 


In this way,
\begin{equation*}
\begin{split}
\E_{M_{*}}[V_1^{\pi_k}(x_0)] = &\E_{M_{*}} \left[ \sum_{h = 1}^H \mathbf{1} \{ x_h = v \} \right]\\
\le & H \cdot \E _{M_{*}} \left[ \mathbf{1} \{ x_H = v \} \right]\\
\le & H \cdot \E_{M_{*}} [ {\mathbb{P}} \left(S_* \in S_k\right) ]\\
\le & H / H_0 \cdot \sum_{h_* =  1}^{H_0} \sum_{\tau = 1}^{h_*} {\mathbb{P}}  \left( s_{\tau, i_{\tau}} \in S_k^{\tau}  \right)\\
=& 2 \sum_{h_* =  1}^{H_0}\sum_{\tau = 1}^{h_*} {\mathbb{P}} \Bigg ( \bigcup_{j \ge 1} \Big \{ r_{\tau} \ge j , s_{\tau, i_{\tau}} \not \in \{ \tilde s _{\tau, i_1}, \tilde s _{\tau, i_2}, \cdots, \tilde s _{{\tau}, i_{j-1}} \}, s_{\tau, i_{\tau}} = \tilde s _{{\tau}, i_{j}} \Big \} \Bigg )\\
=& 2 \sum_{h_* =  1}^{H_0}\sum_{\tau = 1}^{h_*} \sum_{j \ge 1} \PR{r_{\tau} \ge j} \cdot \PR{s_{\tau, i_{\tau}} \not \in \{ \tilde s _{\tau, i_1}, \tilde s _{\tau, i_2}, \cdots, \tilde s _{{\tau}, i_{j-1}} \}, s_{\tau, i_{\tau}} = \tilde s _{{\tau}, i_{j}} | r_{\tau} \ge j }\\
\end{split}
\end{equation*}

Now suppose that we know $r_{\tau} \ge j.$ Noticing that $ s_{\tau, i_{\tau}} \sim \rm{Unif}$$( \{ s_{\tau, i} | i \in [d_0] \} ),$ we have
\begin{equation*}
\begin{split}
&\PR{s_{\tau, i_{\tau}} \not \in \{ \tilde s _{\tau, i_1}, \tilde s _{\tau, i_2}, \cdots, \tilde s _{{\tau}, i_{j-1}} \}, s_{\tau, i_{\tau}} = \tilde s _{{\tau}, i_{j}} | r_{\tau} \ge j } \\
=& \prod_{\gamma = 1}^{j - 1} \frac{d_0 - \gamma}{d_0 - \gamma + 1} \cdot \frac{1}{d_0 - j + 1} = \frac{1}{d_0} \\
\end{split}
\end{equation*}
Substituting this into the preceding bound gives
\begin{equation*}
\begin{split}
\E_{M_{*}}[V_1^{\pi_k}(x_0)] \le & 2\sum_{h_* =  1}^{H_0}\sum_{\tau = 1}^{h_*} \sum_{j \ge 1}\PR{r_\tau \ge j}/d_0\\
= & 2/d_0 \sum_{h_* =  1}^{H_0}\sum_{\tau = 1} ^{h_*} \E [r_\tau] = 2/d_0 \sum_{h_* = 1}^{H_0} \E \left [r\right]\\
\le & \E[ N_{\rm{switch}}^k + H ] \cdot 2/d_0 \le \E[ N_{\rm{switch}}^{\rm gl} + H ] \cdot 2/d_0\\
\le & H_0 / 2 \\
\end{split}
\end{equation*}
as $N_{\rm switch}^{\rm gl} \le dH/100$ almost surely and $d \ge 100$. And thus
\begin{equation*}
\begin{split}
&\sup_{M \in \mathcal{M}} \mathbb{E}_{s_1, \mathcal{M}}{\left[ \sum_{k = 1}^{K} V_1^*(s_1) - V_1^{\pi_k}(s_1) \right ]} \\
&\ge KH_0 - \sum_{k = 1}^K \E_{M_{*}}[V_1^{\pi_k}(x_0)]\\\
&\ge KH_0 -KH_0/2\\
& = KH/4\\
\end{split}
\end{equation*}

\paragraph{Remark}
In fact, we can reduce states in $S_0$ and only reserve three states $\{u, v, w \}$ with similar structure: the agent needs to find the correct action set $\{ a_{i_h} \}$ for $h_*$ steps. In this way, we can use identical $d$ actions in each states and thus we provide a tighter lower bound $\Omega(dH).$

\end{document}